\ifdefined\pdfoutput
  \ifdefined\XeTeXversion\else
    \pdfoutput=1
  \fi
\fi
\documentclass[11pt]{article}

\usepackage[letterpaper,margin=1in]{geometry}
\usepackage[numbers,compress]{natbib}

\usepackage[utf8]{inputenc}
\usepackage[T1]{fontenc}
\usepackage{amsmath,amssymb,amsfonts,amsthm,mathtools,bm,bbm}
\usepackage{graphicx}
\usepackage{booktabs}
\usepackage{adjustbox}
\usepackage{enumitem}
\usepackage{float}
\usepackage{microtype}
\usepackage{tikz}
\usepackage{url}
\usepackage{xcolor}
\usepackage{hyperref}
\usepackage[capitalize,noabbrev,nameinlink]{cleveref}
\usepackage{algorithm}
\usepackage{algorithmic}
\usetikzlibrary{arrows.meta,positioning}

\hypersetup{
  colorlinks=true,
  linkcolor=black,
  citecolor=black,
  urlcolor=black
}

\allowdisplaybreaks

\newtheorem{theorem}{Theorem}[section]
\newtheorem{lemma}[theorem]{Lemma}
\newtheorem{proposition}[theorem]{Proposition}
\newtheorem{corollary}[theorem]{Corollary}
\theoremstyle{definition}
\newtheorem{definition}[theorem]{Definition}

\theoremstyle{remark}
\newtheorem{remark}[theorem]{Remark}

\crefname{theorem}{Theorem}{Theorems}
\Crefname{theorem}{Theorem}{Theorems}
\crefname{lemma}{Lemma}{Lemmas}
\Crefname{lemma}{Lemma}{Lemmas}
\crefname{proposition}{Proposition}{Propositions}
\Crefname{proposition}{Proposition}{Propositions}
\crefname{corollary}{Corollary}{Corollaries}
\Crefname{corollary}{Corollary}{Corollaries}
\crefname{definition}{Definition}{Definitions}
\Crefname{definition}{Definition}{Definitions}
\crefname{remark}{Remark}{Remarks}
\Crefname{remark}{Remark}{Remarks}
\crefname{section}{Section}{Sections}
\Crefname{section}{Section}{Sections}
\crefname{subsection}{Section}{Sections}
\Crefname{subsection}{Section}{Sections}
\crefname{algorithm}{Algorithm}{Algorithms}
\Crefname{algorithm}{Algorithm}{Algorithms}

\newcommand{\eqdef}{\ensuremath{\stackrel{\mbox{\upshape\tiny def.}}{=}}}
\newcommand{\ra}[1]{\renewcommand{\arraystretch}{#1}}
\DeclareMathOperator*{\argmax}{arg\,max}
\DeclareMathOperator*{\argmin}{arg\,min}

\renewcommand{\phi}{\varphi}
\newcommand{\eps}{\varepsilon}

\newcommand{\rr}{\mathbb{R}}
\newcommand*\diff{\mathop{}\!\mathrm{d}}

\let\emptyset\varnothing

\newcommand{\N}{\mathbb{N}}

\def\p{\partial}
\def\e{\varepsilon}

\let\And\and
\title{Mixtures of Neural Operators Reduce Active Complexity in Operator Learning}

\author{
Anastasis Kratsios\\
McMaster University and Vector Institute\\
\texttt{kratsioa@mcmaster.ca}
\And
Takashi Furuya\\
Shimane University\\
\texttt{takashi.furuya0101@gmail.com}
\And
Antonio Lara\\
Rice University\\
\texttt{antonio.lara@rice.edu}
\And
Matti Lassas\\
University of Helsinki\\
\texttt{Matti.Lassas@helsinki.fi}
\And
Maarten de Hoop\\
Rice University\\
\texttt{mdehoop@rice.edu}
}

\date{}

\begin{document}

\maketitle

\begin{abstract}
Operator-learning systems are not governed solely by total parameter count; for one query, the relevant bottleneck can be the model that must be loaded and evaluated. We study this distinction for classical neural operators on compact Sobolev subsets through a constructive comparison between routed mixtures of neural operators (MoNOs) and a fixed single-neural-operator construction. The comparison concerns expert-active complexity relative to that baseline, with total stored size and routing search accounted separately. A MoNO routes each input function through a tree to one expert. Our main theorem shows that every scalar uniformly continuous nonlinear operator with bounded output Sobolev radius on the approximation set admits a MoNO approximation whose active expert has smaller depth, width, and rank scaling than the analyzed single-neural-operator construction; for Lipschitz targets these expert quantities are bounded by $\mathcal{O}(\varepsilon^{-1})$. The theorem turns localization into an operator-level accounting of active expert size, routing depth, and number of experts. We also prove a quantitative universal approximation theorem for the underlying neural-operator architecture, with explicit dependence on compact-set diameter and modulus of continuity.
\end{abstract}

\section{Introduction}
\label{s:Introduction}

Neural operators (NOs) are a central model class in scientific machine learning because they learn maps between function spaces rather than maps between fixed finite-dimensional vectors. They have produced strong empirical results in physics, finance, control, game theory, and inverse problems; see, for example, \cite{Chen2_OG_NeuralOperators_IEEE_1995,lu2021learning,li2021fourier,kovachki2021neural,korolev2022two,herrmann2022neural,adcock2022near,raonic2024convolutional,de2022deep,pmlr-v202-molinaro23a,andrade2023poissonnet}. Variants also encode structure such as representation equivariance, invertibility, and causality \cite{bartolucci2024representation,furuya2023globally,galimberti2022designing,mccabe2023towards,acciaio2023designing}. From a constructive approximation-theoretic viewpoint, however, total parameter count is not the only relevant resource: one must also ask how much of the model is active on a single query. For the finite-rank single-NO construction analyzed below, the displayed depth bound can become very large on compact subsets of function space.

The main theorem is best viewed as a constructive comparison theorem. Relative to the single-NO bound of \Cref{prop:quantitative_UAT__ON}, routing improves the size of the expert active on one query. The accounting separates expert-active size from total stored size and routing resources, and the comparison is made against this explicit constructive baseline.

This paper asks whether routing can reduce that \emph{active} burden without shrinking the target class. In finite dimensions, classical approximation theory predicts a parameter count polynomial in $\eps^{-1}$, with exponent tied to ambient dimension \cite{yarotsky2018optimal,petersen2018optimal,herrmann2022constructive,zhang2023deep}. In the operator setting, the constructive worst-case bounds used here are much more severe: the depth of the analyzed single-NO construction can grow exponentially in the effective accuracy scale. We use this construction as a fixed reference point for the routed comparison; shallow universal operator approximators and FNO universality are well known \cite{Chen2_OG_NeuralOperators_IEEE_1995,lu2021learning,KovachkiLanthalerMishra_UniFNO_JMLR_2021}.

Several alternatives weaken this barrier by changing the approximation class or the target regime: one can restrict to highly regular targets \cite{marcati2023exponential}, impose low-complexity structure \cite{adcock2022near,herrmann2022neural,PCANetErrorBounds_JMLR_2023}, use kernel or random-feature operator learners \cite{batlle2024kernel,nelsen2024operator}, exploit problem-specific numerical structure \cite{lanthaler2023curse,lanthaler2025parametric}, hardwire transport and relaxation physics \cite{larabenitez2025neurde}, or enlarge the activation class \cite{yarotsky2021elementary,yarotsky2020phase,shen2021deep,jiao2023deep}. We instead stay in the classical activation regime and obtain the benefit from localization and routing. The point is not merely that a compact set admits a cover; it is that routing converts such localization into an explicit operator-level trade-off among local diameter, active expert size, routing depth, and number of experts.

\paragraph{Structured comparison class.}
The paper is complementary to regimes where a single global model already has algebraic total complexity, including Barron-type regularity, holomorphic parametric-PDE maps, PCA- or representation-adapted surrogates, kernel and random-feature methods, and other low-complexity settings \cite{barron1993universal,adcock2022near,herrmann2022neural,marcati2023exponential,PCANetErrorBounds_JMLR_2023,batlle2024kernel,nelsen2024operator}. Those results address structure-exploiting total-complexity guarantees, while our focus is expert-active complexity in a worst-case constructive setting where no low-dimensional or analytic structure is assumed.

Our main object is a \emph{mixture of neural operators} (MoNO): a rooted decision tree routes an input function to one leaf, and that leaf stores a specialized neural operator. Instead of approximating the target on all of $K$ with one global NO, we approximate it on smaller local regions. The total parameter count can still be large, but only one expert is active for a given input. Thus the approximation burden shifts from one monolithic operator to the number of experts and the routing structure. This is the operator-learning analogue of sparse activation in modern mixture-of-experts systems such as Switch Transformers and Mixtral \cite{fedus2022switch,jiang2024mixtral,shazeer2017outrageously,lepikhin2021gshard,barham2022pathways,zhou2022mixture}. The theorem below makes this shift quantitative.

\begin{figure*}[t]
\centering
\definecolor{nmi_blue}{RGB}{0, 114, 178}       % Deep professional blue
\definecolor{nmi_accent}{RGB}{213, 94, 0}      % Vermillion for highlighting active paths
\definecolor{nmi_bg}{RGB}{248, 249, 250}       % Very light gray for panel backgrounds
\definecolor{nmi_border}{RGB}{206, 212, 218}   % Subtle gray for panel borders
\definecolor{nmi_text}{RGB}{51, 51, 51}        % Dark gray for softer text readability
\definecolor{nmi_gray}{RGB}{173, 181, 189}     % Muted gray for inactive experts

\begin{adjustbox}{max width=\textwidth,center}
\begin{tikzpicture}[
  >=Latex,
  panel_half/.style={
    draw=nmi_border,
    line width=1pt,
    rounded corners=8pt,
    fill=nmi_bg,
    minimum width=7.2cm,
    minimum height=4.8cm
  },
  panel_full/.style={
    draw=nmi_border,
    line width=1pt,
    rounded corners=8pt,
    fill=nmi_bg,
    minimum width=15.0cm,
    minimum height=3.8cm
  },
  panellabel/.style={
    font=\Large\bfseries\sffamily,
    text=black,
    anchor=north west
  },
  stage/.style={
    draw=nmi_blue!80!black,
    line width=1pt,
    rounded corners=4pt,
    fill=nmi_blue!10,
    text=nmi_text,
    minimum width=1.1cm,
    minimum height=0.7cm,
    align=center,
    font=\small\sffamily
  },
  stage_active/.style={
    draw=nmi_accent,
    line width=1.2pt,
    rounded corners=4pt,
    fill=nmi_accent!10,
    text=nmi_text,
    minimum width=1.1cm,
    minimum height=0.7cm,
    align=center,
    font=\small\sffamily
  },
  stage_inactive/.style={
    draw=nmi_gray,
    line width=1pt,
    rounded corners=4pt,
    fill=nmi_gray!15,
    text=nmi_gray!80!black,
    minimum width=1.1cm,
    minimum height=0.7cm,
    align=center,
    font=\small\sffamily
  },
  title/.style={
    font=\large\bfseries\sffamily, 
    align=center,
    text=nmi_text
  },
  note/.style={
    align=center, 
    text width=5.6cm, 
    font=\small\sffamily,
    text=nmi_text!90!black
  },
  note_left/.style={
    align=left, 
    text width=4.8cm, 
    font=\small\sffamily,
    text=nmi_text!90!black
  },
  pill/.style={
    draw=nmi_blue!60,
    line width=0.8pt,
    rounded corners=12pt,
    fill=white,
    text=nmi_blue!90!black,
    inner xsep=10pt,
    inner ysep=4pt,
    font=\footnotesize\bfseries\sffamily
  },
  pill_accent/.style={
    draw=nmi_accent!80,
    line width=0.8pt,
    rounded corners=12pt,
    fill=white,
    text=nmi_accent,
    inner xsep=10pt,
    inner ysep=4pt,
    font=\footnotesize\bfseries\sffamily
  },
  arrow_style/.style={
    ->, 
    line width=1.2pt, 
    draw=nmi_blue!80!black
  },
  arrow_active/.style={
    ->, 
    line width=1.4pt, 
    draw=nmi_accent
  },
  arrow_inactive/.style={
    ->, 
    line width=1pt, 
    dashed,
    draw=nmi_gray
  }
]

% ==========================================
% PANEL A: Single NO
% ==========================================
% Center at (0,0). Left edge is -3.6, Right edge is 3.6.
\begin{scope}[shift={(0,0)}]
  \node[panel_half] (pA) at (0,0) {};
  \node[panellabel] at ([xshift=10pt, yshift=-10pt]pA.north west) {a};
  
  \node[title] at (0, 1.6) {Single NO};
  
  % Nodes (G is drawn larger to represent full size)
  \node[stage] (u1) at (-2.2, 0.4) {$u$};
  \node[stage, minimum width=1.4cm, minimum height=1.0cm, fill=nmi_blue!15] (g1) at (0, 0.4) {$G$};
  \node[stage] (out1) at (2.2, 0.4) {$G(u)$};
  
  \draw[arrow_style] (u1) -- (g1);
  \draw[arrow_style] (g1) -- (out1);
  
  \node[note] at (0, -0.9) {One global operator handles the entire compact approximation set.};
  
  \node[pill] at (0, -1.9) {Active weights per query: all of $G$};
\end{scope}

% ==========================================
% PANEL B: MoNO Inference
% ==========================================
% Center at (7.8,0). Left edge is 4.2, Right edge is 11.4. Gap is 0.6cm.
\begin{scope}[shift={(7.8,0)}]
  \node[panel_half] (pB) at (0,0) {};
  \node[panellabel] at ([xshift=10pt, yshift=-10pt]pB.north west) {b};
  
  \node[title] at (0, 1.6) {MoNO Inference};
  
  % Nodes (Gl is drawn smaller and accented to show routing)
  \node[stage] (u2) at (-2.6, 0.4) {$u$};
  \node[stage, draw=nmi_text, fill=white] (tree) at (-0.9, 0.4) {$\mathcal{T}$};
  \node[stage_active, minimum width=1.0cm, minimum height=0.6cm] (leaf) at (0.9, 0.4) {$G_\ell$};
  \node[stage] (out2) at (2.6, 0.4) {$G_\ell(u)$};
  
  \draw[arrow_style] (u2) -- (tree);
  \draw[arrow_active] (tree) -- (leaf) node[midway, above, font=\scriptsize\color{nmi_accent}, yshift=2pt] {routes};
  \draw[arrow_active] (leaf) -- (out2);
  
  \node[note] at (0, -0.9) {Routing localizes the approximation problem and activates only \textbf{one} expert on each query.};
  
  \node[pill_accent] at (0, -1.9) {Active experts per query: $1$};
\end{scope}

% ==========================================
% PANEL C: System-Level Cost (Wide Layout)
% ==========================================
% Center at (3.9, -4.8). Left edge is -3.6, Right edge is 11.4.
% Matches exactly with the outer boundaries of A and B.
\begin{scope}[shift={(3.9,-4.8)}]
  \node[panel_full] (pC) at (0,0) {};
  \node[panellabel] at ([xshift=10pt, yshift=-7pt]pC.north west) {c};
  
  % Left side: Text explanation
  \node[title, anchor=west] at (-7.0, 1.0) {System-Level Cost};
  \node[note_left, anchor=west] at (-7.0, -0.4) {The routed system may still store many experts across its leaves, even when only one expert is active.};
  
  % Middle: Tree Structure
  \node[stage, draw=nmi_text, fill=white, minimum width=1.6cm, minimum height=0.7cm] (root) at (0.5, 0.8) {Router $\mathcal{T}$};
  
  \node[stage_inactive] (child1) at (-1.5, -0.7) {$G_1$};
  \node[stage_active] (child2) at (0.5, -0.7) {$G_2$};
  \node[stage_inactive] (child3) at (2.5, -0.7) {$G_3$};
  
  \draw[arrow_inactive] (root.south west) -- (child1.north);
  \draw[arrow_active] (root.south) -- (child2.north);
  \draw[arrow_inactive] (root.south east) -- (child3.north);
  
  % Right side: Stacked pills
  \node[pill] at (5.2, 0.5) {Stored experts: $\Lambda=v^h$};
  \node[pill_accent] at (5.2, -0.6) {Active cost $\neq$ Total cost};
\end{scope}

\end{tikzpicture}
\end{adjustbox}
\caption{\textbf{Qualitative complexity picture.} \textbf{a}, A single neural operator activates one global model on every query. \textbf{b}, A Mixture of Neural Operators (MoNO) routes the input to a specific, smaller expert ($G_\ell$). The primary benefit is a smaller \emph{active} model at evaluation time. \textbf{c}, Total stored size is tracked separately because the routing tree $\mathcal{T}$ may contain many experts across its leaves, decoupling total memory cost from active computational cost.}
\label{fig:complexity}
\end{figure*}

\paragraph{Contributions.}
\begin{enumerate}[leftmargin=1.5em]
\item We prove a routed universal approximation theorem for MoNOs. In the scalar setting stated below, every uniformly continuous nonlinear operator on a compact Sobolev subset can be uniformly approximated by a routed ensemble of neural operators; relative to the constructive single-NO approximation analyzed in this paper, the expert activated on one query is substantially smaller, and in the Lipschitz case the active depth, width, and rank are bounded by $\mathcal{O}(\eps^{-1})$.
\item We derive a quantitative universal approximation theorem for the underlying classical neural-operator architecture. This is an independent secondary result: it isolates how the complexity depends on the diameter of the compact approximation set and on the modulus of continuity, and it also provides the local approximation ingredient behind the routed construction.
\item We identify active complexity as a distinct theoretical resource in operator learning. The analysis separates active expert size, routing depth, and the total number of experts $\Lambda=v^h$, clarifying which quantities routing improves and which costs it pays. The appendix complements this with an inverse-problem case study illustrating why localization can matter under weak stability, and with the observation that a classical neural operator is a degenerate one-expert MoNO.
\item We include two controlled Burgers experiments: a fixed-viscosity control where routing provides no gain, and a variable-viscosity case where learned routing improves active-complexity error. These experiments test the theorem's conditional prediction in a PDE-generated setting.
\end{enumerate}

\paragraph{Perspective.}
Within the constructive framework developed here, routing reduces the expert active on one query relative to \Cref{prop:quantitative_UAT__ON}, while the price is paid in routing depth and number of experts. This distinction matters in settings such as a memory-limited inverse-problem solver, where only the selected expert must be resident in accelerator memory while inactive experts can remain off-device. A naive $\varepsilon$-net argument would only produce local subproblems; the routed architecture also accounts for the routing resources needed to realize them.

\paragraph{Organization.}
\Cref{s:Prelim} introduces the approximation setting, the neural-operator architecture, and the MoNO routing model. \Cref{s:MainResult} states the main theorem and the classical NO approximation bound. \Cref{sec:controlled_experiment} gives two controlled Burgers experiments illustrating the active-versus-total complexity trade-off and its fixed-operator control. \Cref{sec:discussion} explains the proof strategy, the tree construction, and the main modeling trade-offs. The appendices collect the detailed proofs, the routing pseudocode, the inverse-problem application, and the reduction from classical NOs to trivial MoNOs.

\section{Problem Setup and Architectures}
\label{s:Prelim}

This section fixes only the ingredients needed for the main theorem: a compact Sobolev approximation set, the constructive single neural-operator architecture used in the comparison, and the routed MoNO model. Readers mainly interested in the theorem may keep in mind the following picture: a MoNO is a decision tree that sends each input to one expert neural operator.

Let $d^{(1)}, d^{(2)}, d_{in}, d_{out} \in \N$, set $d_i \eqdef d^{(i)}$ for $i=1,2$, and define $D_i \eqdef [0,1]^{d_i}$. We work with Sobolev spaces embedded in ambient $L^2$ spaces, and approximation errors are measured in the relevant $L^2$ norm. For a uniformly continuous map $G:(H,\|\cdot\|_H)\to(H',\|\cdot\|_{H'})$ between separable Hilbert spaces, we write $\omega:[0,\infty)\to[0,\infty)$ for a monotone modulus of continuity satisfying
\[
\|G(u)-G(\tilde u)\|_{H'} \le \omega\bigl(\|u-\tilde u\|_H\bigr)
\qquad \text{for all } u,\tilde u\in H.
\]

\subsection{Sobolev Approximation Sets}

Let $D\subseteq\rr^d$ be compact, let $s>0$, and let $H^s(D)$ denote the usual Sobolev space. For noninteger $s$, writing $m=\lfloor s\rfloor$, we use the equivalent norm
\[
\|f\|_{H^s(D)}
\eqdef
\Bigl(
  \sum_{|\alpha|\le m}\|\partial^\alpha f\|_{L^2(D)}^2
\Bigr)^{1/2}
+
|f|_{H^s(D)},
\]
where
\[
|f|^2_{H^s(D)}
\eqdef
\sum_{|\alpha|=m}
\int_{D\times D}
\frac{|\partial^\alpha f(x)-\partial^\alpha f(y)|^2}{|x-y|^{2(s-m)+d}}
\diff x\,\diff y.
\]
For integer $s$, we use the standard weak-derivative Sobolev norm with derivatives up to order $s$.
For vector-valued inputs we write $H^s(D)^{d_{in}}$ and endow it with the ambient norm $\|\cdot\|_{L^2(D)^{d_{in}}}$.

The compact approximation sets considered throughout the paper are either the centered Sobolev ball
\begin{equation}
\label{eq:compact_set_Sobolev_type}
K
\eqdef
\left\{
f\in H^s(D)^{d_{in}}:\,
\|f\|_{H^s(D)^{d_{in}}}\le R
\right\},
\end{equation}
or, more generally, any closed nonempty subset
\begin{equation}
\label{eq:compact_set_Sobolev_type__noncentered}
K=\overline K
\subseteq
\left\{
f\in H^s(D)^{d_{in}}:\,
\|f-f_0\|_{H^s(D)^{d_{in}}}\le R
\right\},
\end{equation}
for some $f_0\in H^s(D)^{d_{in}}$. By the Rellich--Kondrashov theorem, such sets are compact in $L^2(D)^{d_{in}}$.

\subsection{Neural Network and Neural Operator Models}
\label{s:Neural-networks}

\begin{definition}[Multilayer perceptron]
\label{def:MLPs}
Let $\sigma\in C(\rr)$ and let $d^{(1)},d^{(2)}\in\N$. For a depth $J\in\N_{\ge 0}$ and widths $[d]\eqdef(d_0\eqdef d^{(1)},\dots,d_J\eqdef d^{(2)})$, an MLP with activation $\sigma$ is a finite composition of affine maps and coordinatewise activations,
\begin{equation}
\begin{aligned}
\hat f_\theta(x) &= x^{(J)} + c, \\
x^{(j+1)} &= \sigma\bullet\bigl(A^{(j)}x^{(j)} + b^{(j)}\bigr),
\qquad j=0,\dots,J-1,
\end{aligned}
\end{equation}
where $(A^{(j)},b^{(j)})\in \rr^{d_{j+1}\times d_j}\times \rr^{d_{j+1}}$, $c\in\rr^{d_J}$, the input is $x^{(0)}=x$, and $\sigma\bullet y \eqdef (\sigma(y_i))_i$ denotes coordinatewise activation.
\end{definition}

We write $\mathcal{NN}^{\sigma}_{[d]}$ for this class of networks with widths $[d]$, and $\mathcal{NN}^{\sigma}_{\Delta,W:d^{(1)},d^{(2)}}$ for the union over all such networks with depth at most $\Delta$ and width at most $W$.

\noindent The single-model baseline compared against MoNOs is the following finite-rank neural-operator architecture.
\begin{definition}[Neural operator]
\label{def:neural-operator-v2}
Let $\boldsymbol d\eqdef[d_1,d_2,d_{in},d_{out}]$ and let $L,w,\Delta,N\in\N_{\ge 0}$ with $\max(d_1,d_2,d_{in},d_{out})\le w$. A neural operator
\[
G:(H^{s_1}(D_1)^{d_{in}},\|\cdot\|_{L^2(D_1)^{d_{in}}})
\to
(H^{s_2}(D_2)^{d_{out}},\|\cdot\|_{L^2(D_2)^{d_{out}}})
\]
is given by
\begin{equation}
\label{eq:NORep}
\begin{aligned}
(G(u))(x) &= (K_N^{(L+1)}u_{L+1})(x) + b^{(L+1)}(x), \qquad
u_1(x) = (K_N^{(0)}u)(x) + b^{(0)}(x), \\
u_{\ell+1}(x) &= \tanh\bullet\bigl(W^{(\ell)}u_\ell(x)+b^{(\ell)}\bigr),
\qquad \ell=1,\dots,L,
\end{aligned}
\end{equation}
where $q_1,\dots,q_{L+1}\le w$ are channel widths, $q_{L+1}=d_{out}$, $u_\ell(x)\in\rr^{q_\ell}$, $W^{(\ell)}\in\rr^{q_{\ell+1}\times q_\ell}$, $b^{(\ell)}\in\rr^{q_{\ell+1}}$, and the nonlocal maps
\begin{subequations}
\label{eq:non-local-operator}
\begin{align}
K_N^{(0)}u(x)
&=
\sum_{m,n=1}^N C_{m,n}^{(0)}(u,\varphi_m)_{L^2(D_1)}\psi_n(x),
\label{non-local-operator-1} \\
K_N^{(L+1)}u'(x)
&=
\sum_{m,n=1}^N C_{m,n}^{(L+1)}(u',\psi_m)_{L^2(D_2)}\psi_n(x)
\label{non-local-operator-2}
\end{align}
\end{subequations}
have rank $N$. Here $b^{(0)}$ and $b^{(L+1)}$ are coordinate networks on $D_2$ with width at most $w$, depth at most $\Delta$, and output dimensions $q_1$ and $d_{out}$, respectively.
\end{definition}

We denote this class by $\mathcal{NO}^{\tanh}_{N,w,L,\Delta,\boldsymbol d}$. For a neural operator $G$ in this class, let $P(G)$ denote its number of trainable parameters. Then
\begin{equation}
\label{eq:complexity_NO}
P(G)
\le
\underbrace{L\,w(w+1)}_{W^{(\ell)},\,b^{(\ell)}}
+
\underbrace{2N^2w^2}_{K_N^{(0)},\,K_N^{(L+1)}}
+
\underbrace{2\Delta\,w(w+1)}_{b^{(0)},\,b^{(L+1)}}
\eqdef
\bar P_{\Delta,w,L,N,\boldsymbol d}.
\end{equation}
We use $\bar P_{\Delta,w,L,N,\boldsymbol d}$ as the architecture-level upper bound that appears in the main complexity estimates.

\begin{remark}[Nonlocal operators]
\label{rem:NLOs}
The finite-rank maps in \eqref{eq:non-local-operator} are truncations of integral operators
\[
u \mapsto \left(x \mapsto \int_{\Omega_{\ell}} k^{(\ell)}(x,y)u(y)\,\diff y\right),
\]
where $\Omega_0=D_1$ for the input lifting map and $\Omega_{L+1}=D_2$ for the output map, with matrix-valued kernels of compatible channel dimensions. Choosing orthonormal bases $\{\varphi_n\}$ and $\{\psi_n\}$ and truncating the corresponding series expansion yields $K_N^{(0)}$ and $K_N^{(L+1)}$. This is the sense in which our architecture retains nonlocal structure while keeping the model finite-rank and explicitly countable.
\end{remark}

\begin{remark}[Integral operators in hidden layers]
Including integral operators in the hidden layers, as in some NO variants, is not required for the universality result used here. The hidden-layer operations in \eqref{eq:NORep} already provide the approximation mechanism needed for \Cref{prop:quantitative_UAT__ON}. Adding intermediate integral layers may be useful for modeling, but it does not change the role of the constructive bound studied in this paper \cite{lanthaler2023neural,KovachkiLanthalerMishra_UniFNO_JMLR_2021}.
\end{remark}

\subsection{Mixture of Neural Operators}
\label{s:Prelim__ss:DistributedHypotheses}

\subsubsection{Rooted tree router}
\label{background:rooted_tree}
Instead of a single gating network, we use a finite rooted tree $\mathcal T=(V,E,f_{1,1})$ to organize the experts hierarchically. The leaves of the tree form an approximate net for the compact set $K$, and routing selects a nearest leaf in $L^2$ while the tree records the corresponding root-to-leaf path. This tree structure is the mechanism that converts a global approximation problem into many local ones.

\begin{definition}[Mixture of neural operators]
Let $\boldsymbol d\eqdef[d_1,d_2,d_{in},d_{out}]$ and let $L,W,\Delta,N,h,v\in\N_{\ge 0}$ with $\max(d_1,d_2,d_{in},d_{out})\le W$. A \emph{mixture of neural operators} (MoNO) is a pair $(\mathcal T,\mathcal{NO})$ in which
\[
\mathcal T=(V,E,f_{1,1})
\quad \text{and} \quad
\mathcal{NO}=\{G_{h,\ell}\}_{\ell\in\Lambda},
\]
where $\mathcal T$ is a finite rooted tree of height $h$ and valency $v$ whose nodes are realized by ReLU MLPs in $\mathcal{NN}^{\operatorname{ReLU}}_{L,W:d_1,d_{in}}$, and where each leaf $\ell\in\Lambda$ carries an expert $G_{h,\ell}\in \mathcal{NO}^{\tanh}_{N,W,L,\Delta,\boldsymbol d}$.
\end{definition}

Intuitively, the leaves of $\mathcal T$ form the covering net used for routing, while internal nodes organize those leaf centers through the aggregate hierarchical relation stated in the appendix. Each selected leaf only needs to approximate the target operator on a small local region.

\paragraph{Realization of a MoNO.}
\label{defn:Distributed_Models}
For simplicity, assume $d_{in}=d_{out}=1$. Given $u\in H^{s_1}(D_1)\subset L^2(D_1)$, routing selects a leaf whose center is nearest to $u$ in $L^2(D_1)$ and then follows the corresponding root-to-leaf path in $\mathcal T$. The expert attached to that leaf is evaluated. A pseudocode version is recorded in the appendix as \Cref{alg:realization_mix}. We account explicitly for the tree depth and number of leaves; with the nearest-leaf rule, the routing search uses $\Lambda$ nearest-center comparisons, and the theorem does not assert an optimized nearest-center search procedure. Each comparison is an $L^2$ distance computation on the input discretization, with cost independent of expert depth, width, and rank; this routing-search cost is separate from the single-expert evaluation cost counted by $\operatorname{Activ-Cpl}$.

The \emph{expert-active complexity} of $(\mathcal T,\mathcal{NO})$ is
\[
\operatorname{Activ-Cpl}(\mathcal T,\mathcal{NO})
\eqdef
\max_{\ell\in\Lambda} P(G_\ell)
\le
\bar P_{\Delta,W,L,N,\boldsymbol d},
\]
which counts the expert loaded after routing, not the cost of searching the routing tree. The routing cost is tracked separately through the height $h$ and the leaf count $\Lambda$.
The \emph{total distributed complexity} is
\[
\operatorname{Dist-Cpl}(\mathcal T,\mathcal{NO})
\eqdef
\sum_{\ell\in\Lambda} P(G_\ell)
\le
|\Lambda|\,\bar P_{\Delta,W,L,N,\boldsymbol d}.
\]
The corresponding routing depth is simply the height $h$ of $\mathcal T$.

\begin{remark}[Neural operators are trivial MoNOs]
\label{remark:NO_as_MoNO}
Any classical neural operator can be realized as a one-expert MoNO by using a tree with a single node and routing every input directly to that expert. \Cref{ex:Centralized_as_distributed} records this identification explicitly.
\end{remark}

\section{Main Results}
\label{s:MainResult}

We now approximate a nonlinear operator
\[
G^+:(H^{s_1}(D_1)^{d_{in}},\|\cdot\|_{L^2(D_1)^{d_{in}}})
\to
(H^{s_2}(D_2)^{d_{out}},\|\cdot\|_{L^2(D_2)^{d_{out}}})
\]
uniformly on a compact subset $K$. By continuity on a compact set, $G^+$ admits a modulus of continuity $\omega$; when invoking the quantitative single-NO proposition, we may replace $\omega$ by a concave majorant. For notational simplicity, we state the main theorem in the scalar case $d_{in}=d_{out}=1$.

\paragraph{Theorem overview.}
The theorem reports three quantities: uniform approximation error on $K$, the size of the expert active on one query, and the routing cost through tree depth and leaf count. It is a localization principle relative to the constructive single-NO bound of \Cref{prop:quantitative_UAT__ON}; total stored complexity, routing-search complexity, and other sparse/adaptive NO classes are accounted for separately from this expert-active comparison.

\paragraph{Construction parameters.}
The displayed quantities below are the parameters used by the constructive proof. Let $C_R$ be the constant from \Cref{lem:complete_n_ary_tree_counting}. Define
\[
R_{G,K}
\eqdef
\sup_{u\in K}\|G^+(u)\|_{H^{s_2}([0,1]^{d_2})}
\quad\text{and choose}\quad
N_\varepsilon
\in
\mathcal{O}\!\left(
[\omega^{-1}(\varepsilon)]^{-d_1/s_1}
\vee
\varepsilon^{-d_2/s_2}
\right),
\]
with hidden constants depending on $R$ and $R_{G,K}$. Set
\[
\rho_\varepsilon
\eqdef
\omega^{-1}\!\left(
\frac{\varepsilon}{(2+N_\varepsilon/2)N_\varepsilon}
\right),
\qquad
\delta_\varepsilon
\eqdef
\frac{\rho_\varepsilon^2}{4C_R},
\]
and define the routing depth and leaf count by
\[
h_\varepsilon
\eqdef
\left\lceil
\log_v\!\bigl(2^{\lceil \delta_\varepsilon^{-d_1/s_1}\rceil}-1\bigr)
\right\rceil,
\qquad
\Lambda
\eqdef
v^{h_\varepsilon}.
\]

\begin{theorem}[Universal approximation for MoNOs]
\label{thrm:Main}
Assume the setting of \Cref{s:Prelim__ss:DistributedHypotheses}. Let $K$ be as in \eqref{eq:compact_set_Sobolev_type}, suppose $s_i>d_i$ for $i=1,2$, fix $\varepsilon>0$, and choose a valency $v\in\N_{\ge 2}$. For every uniformly continuous operator
\[
G^+:(H^{s_1}([0,1]^{d_1}),\|\cdot\|_{L^2(D_1)})
\to
(H^{s_2}([0,1]^{d_2}),\|\cdot\|_{L^2(D_2)})
\]
with modulus of continuity $\omega$ and $R_{G,K}<\infty$, there exists a MoNO $(\mathcal T,\mathcal{NO})$ with realization
\[
G:(H^{s_1}([0,1]^{d_1}),\|\cdot\|_{L^2(D_1)})
\to
(H^{s_2}([0,1]^{d_2}),\|\cdot\|_{L^2(D_2)})
\]
such that
\[
\sup_{u\in K}\|G^+(u)-G(u)\|_{L^2([0,1]^{d_2})}\le \varepsilon.
\]
Its expert-active parameter count satisfies
\[
\operatorname{Activ-Cpl}(\mathcal T,\mathcal{NO})
\in
\mathcal{O}\!\left(N_\varepsilon^4\right),
\]
its routing depth is $h_\varepsilon$, and its leaf count is $\Lambda=v^{h_\varepsilon}$.
The router construction is summarized in \Cref{lem:complete_n_ary_tree_counting}, and the full approximate hierarchical relation used in the proof is recorded in \Cref{s:Proof__ss:TreeLemma}. The simplified complexity comparison is summarized in \Cref{tab:thm:quantitative_UAT__ON}.
\end{theorem}

\begin{proof}
See \Cref{s:Proofs__ss:MainResult}.
\end{proof}

\Cref{thrm:Main} is a structural decomposition theorem for operator approximation on compact sets: localization controls the active expert, while the routing tree records the geometric cost of realizing that localization. Within this constructive framework, the depth and memory footprint of the \emph{active} expert can be controlled separately from the total number of experts, but the routed ensemble can still be large. The routing cost made explicit by the theorem is
\[
(h_\varepsilon,\Lambda)=\bigl(h_\varepsilon,v^{h_\varepsilon}\bigr).
\]
Thus the result concerns \emph{active} complexity relative to the constructive single-NO baseline, while total stored parameter count is tracked separately; see \Cref{fig:complexity}.

\begin{corollary}[Lipschitz targets]
\label{cor:lipschitz_takeaway}
Assume the setting of \Cref{thrm:Main}, and suppose that $G^+$ is Lipschitz on $K$. Then the construction in \Cref{thrm:Main} can be chosen so that the active expert has depth, width, and rank in
\[
\mathcal{O}(\varepsilon^{-1}).
\]
The contribution is the active-complexity accounting: within this construction, the rate is achieved by the expert active on one query while routing resources are recorded separately. The constructive single-NO depth bound in \Cref{tab:prop:quantitative_UAT__ON} is much larger as $\varepsilon\downarrow 0$, so the comparison isolates the effect of localization while holding the constructive approximation ingredient fixed.
\end{corollary}

\begin{proof}
If $\omega(t)\lesssim t$, then $\omega^{-1}(\varepsilon)\gtrsim \varepsilon$, so \Cref{thrm:Main} gives $N_\varepsilon\in\mathcal{O}(\varepsilon^{-1})$. The proof of \Cref{thrm:Main} constructs each active expert with depth, width, and rank of order $N_\varepsilon$, which yields the first claim. The second is the upper-bound comparison recorded in \Cref{tab:prop:quantitative_UAT__ON}.
\end{proof}

For readability, set
\[
z_\varepsilon \eqdef \max\{\varepsilon^{-1},[\omega^{-1}(\varepsilon)]^{-1}\}.
\]
We use this shorthand only in the following simplified table, retaining the maximum to cover Lipschitz and weaker moduli uniformly.

\begin{table}[ht]
\ra{1.25}
\centering
\begin{adjustbox}{max width=0.98\linewidth,center}
\begin{tabular}{@{}l|ll@{}}
\cmidrule[0.3ex](){1-3}
\multicolumn{1}{c}{\textbf{Parameter}} & \multicolumn{2}{c}{\textbf{Simplified Constructive Upper Bounds}} \\
\cmidrule[0.2ex](){1-3}
& \textit{MoNO active expert} & \textit{Single NO of \Cref{prop:quantitative_UAT__ON}} \\
Depth $(L)$ &
$\mathcal{O}(z_\varepsilon)$ &
$\mathcal{O}\!\left(
\dfrac{z_\varepsilon}
{
\bigl(
\omega^{-1}(
\varepsilon z_\varepsilon^{-2}
)
\bigr)^{2z_\varepsilon}
}
\right)$ \\
Width $(W)$ &
$\mathcal{O}(z_\varepsilon)$ &
$\mathcal{O}(z_\varepsilon)$ \\
Rank $(N)$ &
$\mathcal{O}(z_\varepsilon)$ &
$\mathcal{O}(z_\varepsilon)$ \\
Bases $(\varphi_\cdot,\psi_\cdot)$ &
\multicolumn{2}{c}{Piecewise polynomial \cite{birman1967piecewise}} \\
\cmidrule[0.3ex](){1-3}
$\#$ experts / leaves &
$\Lambda = v^{h_\varepsilon}$ &
$1$ \\
$h$ (routing depth) &
$h_\varepsilon$; see \Cref{thrm:Main} &
trivial one-node route \\
\cmidrule[0.3ex](){1-3}
\end{tabular}
\end{adjustbox}
\caption{Simplified constructive comparison between the routed MoNO approximation of \Cref{thrm:Main} and the specific single-NO construction of \Cref{prop:quantitative_UAT__ON}. The exact theorem and proposition bounds are stated in the text; the table records the dominant scaling used in the discussion. The single-NO column is the constructive baseline used throughout the paper. The comparison separates the size of the active expert from the size of the routed ensemble. Since $\Lambda=v^h$, the total number of experts can still be large even when the active expert is small.}
\label{tab:thm:quantitative_UAT__ON}
\end{table}

We now record the constructive single-NO baseline used throughout the comparison.

\begin{proposition}[Approximation rates for classical neural operators]
\label{prop:quantitative_UAT__ON}
Let $D_i=[0,1]^{d_i}$ for $i=1,2$, let $K$ be as in \eqref{eq:compact_set_Sobolev_type__noncentered}, and let
\[
\boldsymbol d\eqdef[d_1,d_2,d_{in},d_{out}].
\]
Let
\[
G^+:
(H^{s_1}(D_1)^{d_{in}},\|\cdot\|_{L^2(D_1)^{d_{in}}})
\to
(H^{s_2}(D_2)^{d_{out}},\|\cdot\|_{L^2(D_2)^{d_{out}}})
\]
be uniformly continuous with concave modulus of continuity $\omega$. Assume that the following Sobolev output radius is finite, and set
\[
R_{K,\mathrm{in}}
\eqdef
\sup_{u\in K}\|u\|_{H^{s_1}(D_1)^{d_{in}}},
\qquad
R_{K,\mathrm{out}}
\eqdef
\sup_{u\in K}\|G^+(u)\|_{H^{s_2}(D_2)^{d_{out}}}.
\]
For every $\varepsilon>0$, there exist $N,L,W\in\N_{\ge 0}$, orthonormal bases of $L^2(D_1)$ and $L^2(D_2)$ with basis elements in $H^{s_1}(D_1)$ and $H^{s_2}(D_2)$, respectively, and a neural operator
\[
G\in \mathcal{NO}^{\tanh}_{N,W,L,1,\boldsymbol d}
\]
such that
\[
\sup_{u\in K}\|G^+(u)-G(u)\|_{L^2(D_2)^{d_{out}}}\le \varepsilon.
\]
The resulting rank, depth, and width bounds are listed in \Cref{tab:prop:quantitative_UAT__ON}.
\end{proposition}

\begin{proof}
See \Cref{proof:UAT_ON}.
\end{proof}

\begin{table}[ht]
\ra{1.25}
\centering
\begin{adjustbox}{max width=\linewidth,center}
\begin{tabular}{@{}ll@{}}
\cmidrule[0.3ex](){1-2}
\textbf{Parameter} & \textbf{Upper bounds} \\
\midrule
Depth $(L)$ &
$\mathcal{O}\!\Bigl(
Nd_{out}
\bigl(|D_2|^{1/2}\operatorname{diam}_{L^2(D_1)^{d_{in}}}(K)\bigr)^{Nd_{in}}
\bigl(
\omega^{-1}(
\tfrac{\varepsilon|D_2|^{1/2}}{(2+N d_{in}/2)Nd_{out}}
)
\bigr)^{-2Nd_{in}}
\Bigr)$ \\
Width $(W)$ & $Nd_{in}+Nd_{out}+2$ \\
Rank $(N)$ &
$\mathcal{O}\!\left(
\left(\frac{R_{K,\mathrm{in}}}{\omega^{-1}(\varepsilon)}\right)^{d_1/s_1}
\vee
\left(\frac{R_{K,\mathrm{out}}}{\varepsilon}\right)^{d_2/s_2}
\right)$ \\
Bases $(\varphi_\cdot,\psi_\cdot)$ & Piecewise polynomial \cite{birman1967piecewise} \\
\bottomrule
\end{tabular}
\end{adjustbox}
\caption{Quantitative approximation rates for the classical neural-operator architecture used throughout the paper. The rank bound depends on the Sobolev radii $R_{K,\mathrm{in}}$ and $R_{K,\mathrm{out}}$ defined in \Cref{prop:quantitative_UAT__ON}.}
\label{tab:prop:quantitative_UAT__ON}
\end{table}

\section{Controlled Burgers Experiments}
\label{sec:controlled_experiment}

The theorem makes a conditional prediction: routing can reduce the size of the expert evaluated on one query when the operator family decomposes into localized regimes, while a homogeneous family already captured by one global model should show no routing advantage. We test this active-versus-total trade-off directly with two controlled Burgers experiments. In both experiments, inputs are random periodic Fourier initial conditions
\[
u_0(x)=\sum_{j=1}^{8} a_j\sin(2\pi jx)+b_j\cos(2\pi jx),
\]
sampled on a grid of size $64$. Targets are generated by an explicit finite-volume solver for the one-dimensional viscous Burgers equation at final time $T=0.4$. The route label $r(u_0)\in\{0,1,2,3\}$ is the empirical quartile of the initial-condition steepness $\|\partial_xu_0\|_\infty$. In the homogeneous control, every input is evolved with viscosity $\nu=0.01$, so the route labels do not correspond to different PDE operators. In the heterogeneous experiment, the four steepness quartiles are evolved with viscosities $(0.05,0.02,0.008,0.003)$. The heterogeneous target is therefore PDE-generated rather than a soft mixture of experts, but it is deliberately structured so that localization has a meaningful regime label. This defines a single constructed input-output map on the sampled distribution; the viscosity-by-steepness coupling is a controlled route-recovery task for a structural input regime, rather than a claim about inherent heterogeneity in physical Burgers data.

Each experiment compares four models. The active-matched global neural operator has the same size as one MoNO expert. The total-matched global neural operator is widened until its parameter count matches the stored learned-router MoNO. The oracle MoNO uses the steepness label $r(u_0)$ to choose the expert during both training and testing; it tests whether the chosen partition creates easier local subproblems. The learned-router MoNO uses the same expert architecture, trains the router by supervised classification of $r(u_0)$, and evaluates only the expert selected by the learned router at test time. Thus the learned-router row is a supervised-routing measurement. The active-matched global model and each MoNO expert use a 1D Fourier-style neural operator \cite{li2021fourier} with width $16$, $10$ Fourier modes, and $4$ spectral blocks; the total-matched global model uses width $33$. Full training and routing details are in \Cref{app:experiment_diagnostics}.

\begin{figure*}[t]
\centering
\includegraphics[width=\textwidth]{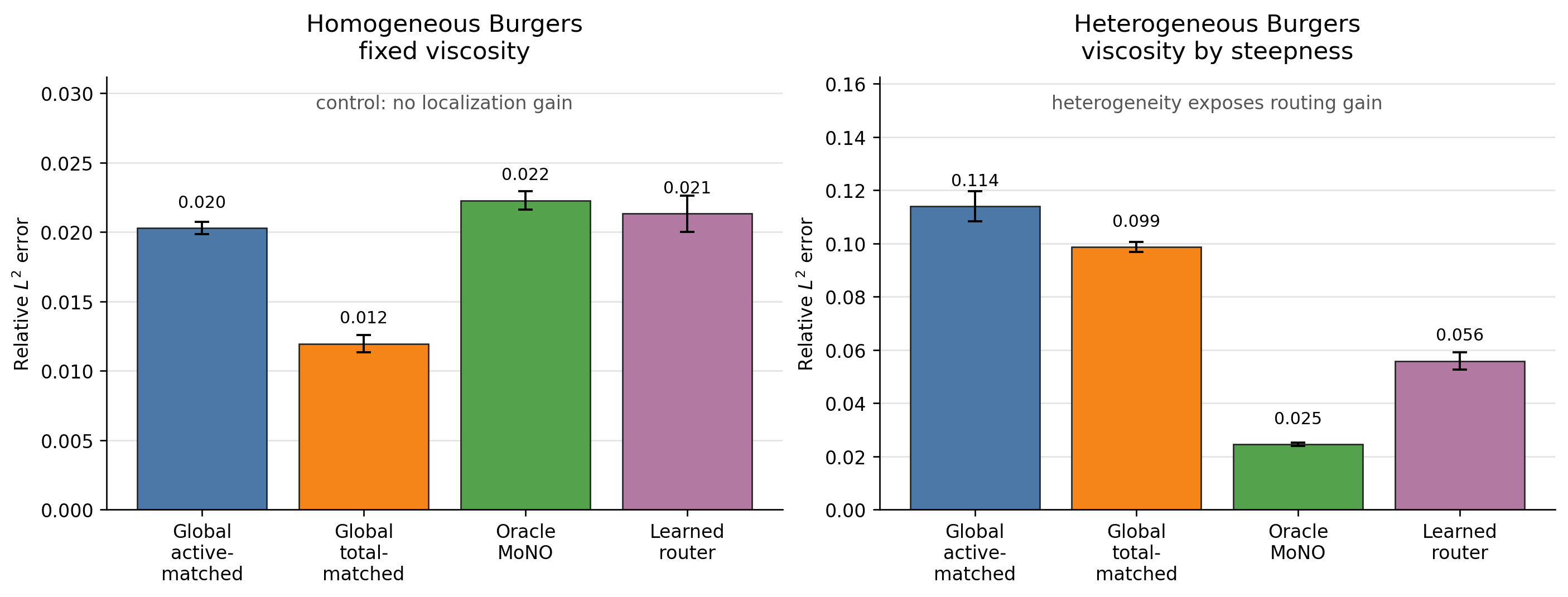}
\caption{\textbf{Controlled Burgers experiments.} Bars show mean relative $L^2$ test error over three seeds, with standard-deviation error bars. Left: on a homogeneous fixed-viscosity Burgers operator, routing gives no gain over the active-matched global model and the larger global baseline is best. Right: on a heterogeneous variable-viscosity Burgers family, learned routing improves over both global baselines while keeping the active expert at the active-matched scale.}
\label{fig:active_complexity_experiment}
\end{figure*}

\Cref{fig:active_complexity_experiment} shows the intended contrast. In the homogeneous control, the active-matched global NO attains relative $L^2$ error $0.0203\pm0.0004$, the oracle MoNO attains $0.0223\pm0.0007$, the learned-router MoNO attains $0.0213\pm0.0013$, and the larger total-matched global NO is best at $0.0119\pm0.0006$. This negative control confirms that routing is useful here because of operator heterogeneity, not architecture alone. In the heterogeneous experiment, the active-matched global NO uses $21{,}905$ active parameters and attains $0.1140\pm0.0056$, while the total-matched global NO uses $92{,}863$ active parameters and attains $0.0986\pm0.0019$. The oracle MoNO attains $0.0245\pm0.0007$, showing that the steepness-viscosity partition creates substantially easier local problems. The learned-router MoNO attains $0.0558\pm0.0033$ with route accuracy $0.858\pm0.011$, using $21{,}905$ active expert parameters, $27{,}477$ active parameters including the router, and $93{,}192$ total stored parameters. Thus the heterogeneous gain is tied to active complexity rather than total stored size. Full parameter accounting is in \Cref{tab:burgers_active_complexity}.

\section{Proof Strategy and Discussion}
\label{sec:discussion}
\label{s:Main__ss:Discussion}

The proof is organized around an explicit constructive baseline. \Cref{prop:quantitative_UAT__ON} gives a single-NO approximation theorem whose complexity deteriorates with the diameter of the compact approximation set. \Cref{lem:complete_n_ary_tree_counting} supplies a tree-indexed leaf cover with explicit routing depth and leaf count. Combining them yields a routed architecture whose \emph{expert-active} complexity is governed by local diameter, while the cost is carried by the tree and expert ensemble.

\paragraph{What the routing lemma contributes.}
A plain $\varepsilon$-net argument only partitions $K$ into local problems. The routing lemma packages that localization into a tree-indexed cover and ties the local scale to explicit routing depth and leaf count. The full statement, including router-size bounds, is in the appendix.

\paragraph{Scope of the comparison.}
The theorem is a constructive statement about how routing changes one approximation strategy. Its baseline is the explicit single-NO construction of \Cref{prop:quantitative_UAT__ON}; within that framework, routing replaces one global problem by local ones, reducing the expert active on a query while exposing the cost through routing depth and expert count $\Lambda=v^h$.

This comparison is complementary to regimes where shallow universal constructions, FNO/PCA-based surrogates, kernel methods, random features, structure-aware PDE architectures, or intrinsic Barron/analytic/low-dimensional structure already give algebraic total complexity \cite{barron1993universal,Chen2_OG_NeuralOperators_IEEE_1995,lu2021learning,KovachkiLanthalerMishra_UniFNO_JMLR_2021,adcock2022near,herrmann2022neural,marcati2023exponential,PCANetErrorBounds_JMLR_2023,batlle2024kernel,nelsen2024operator,lanthaler2025parametric}. Those results exploit additional structure to control total complexity; the routed bound instead addresses expert-active complexity in the worst-case Sobolev/uniform-continuity setting.

Because \Cref{prop:quantitative_UAT__ON} supplies both the single-NO baseline and the local experts, \Cref{thrm:Main} isolates localization while holding the constructive approximation ingredient fixed. Replacing the baseline by a different shallow universal construction would answer a different comparison question.

\paragraph{Auxiliary remarks.}
For general compact domains, the appendix proves a domain-extension lemma that transfers the argument to ambient cubes without changing the modulus of continuity; see \Cref{prop:extension_lemma}. It also records the routing pseudocode, inverse-problem case study, and one-expert reduction. We use standard activations rather than super-expressive families \cite{yarotsky2020phase,yarotsky2021elementary,shen2021deep,jiao2023deep}, so the gain comes from routing and localization.

\appendix

\section{Detailed Proofs}
\label{app:technical-proofs}
\label{s:Proof_Details}
The proofs rely on several key notions from constructive approximation. We begin by recalling the concept of {\it metric entropy} for a compact subset of $L^2(D)$. For each $k\in \mathbb{N},$ this is quantified by its $k$th entropy number, defined as
\[
        e_k(K)
    \eqdef
        \inf
        \left\{\delta>0:\,
            \exists f_1,\dots,f_{2^k-1}\in L^2({D}),
            \quad
            \max_{f\in K}\,
                \min_{i=1,\dots,2^k-1}\,
                \|
                    f-f_i
                \|_{L^2({D})}
            <
                \delta
        \right\}.
\]

\noindent See \cite[Chapter 13]{LorentzGoliteschekMakovoz_1996_CAAdvProbsBook} for further details.
We also recall the definition of the \textit{Kolmogorov linear $N$-width}. For a subset $A$ of a Banach space $X$ and $N\in \mathbb{N}_{\ge0}$,
\[
        d_N(A;X)
    \eqdef
        \inf_{L\in \mathcal{L}_N}
        \,
        \sup_{x\in A}\,\inf_{y\in L}
        \,
            \|x-y\|_{X},
\]
where $\mathcal{L}_N$ is the set of $N$-dimensional linear subspaces of $X$; see \cite[Chapter 1.1]{pinkus2012n}.

The quantitative argument proceeds in three stages. First, we build constructive piecewise-polynomial approximation spaces for the compact input set~\eqref{eq:compact_set_Sobolev_type__noncentered} and for its image under $G^+$. Next, following the finite-rank reduction used in \cite[Proof of Theorem 11]{kovachki2021neural}, we project the operator-learning problem to a finite-dimensional surrogate between Euclidean spaces. Finally, we apply the quantitative universal approximation theorem of~\cite[Proposition 53]{kratsios2022universal} for smooth activations and then choose the approximation rank $N$ so that the projection and network errors both stay below the target tolerance.

\subsection[Proof of the single-NO approximation proposition]{Proof of Proposition~\ref{prop:quantitative_UAT__ON}}\label{proof:UAT_ON}
\paragraph{\textit{Step 1: Constructive low-dimensional subspaces}}
The quantitative approximation argument only needs constructive approximation spaces for the compact input set $K$ and its image $G^+(K)$. Exact recentering is unnecessary here.
Let
\[
R_{K,\mathrm{in}}
\eqdef
\sup_{u\in K}\|u\|_{H^{s_1}(D_1)^{d_{in}}},
\qquad
R_{K,\mathrm{out}}
\eqdef
\sup_{u\in K}\|G^+(u)\|_{H^{s_2}(D_2)^{d_{out}}}.
\]
The input radius is finite by the definition of $K$, and the output radius is finite by the assumption in \Cref{prop:quantitative_UAT__ON}.
By \cite[Theorems 3.3 and 5.1]{birman1967piecewise}, for every $N\in\mathbb{N}$ there exist $N$-dimensional scalar piecewise-polynomial subspaces
\[
V_{N,K}^{(1)}\subseteq L^2(D_1)
\qquad\text{and}\qquad
W_{N,K}^{(2)}\subseteq L^2(D_2)
\]
whose componentwise products satisfy
\begin{equation}
\label{eq:constructive_input_width}
\sup_{\|u\|_{H^{s_1}(D_1)^{d_{in}}}\le 1}
\inf_{y\in (V_{N,K}^{(1)})^{d_{in}}}
\|u-y\|_{L^2(D_1)^{d_{in}}}
\lesssim
N^{-s_1/d_1},
\end{equation}
and
\begin{equation}
\label{eq:constructive_output_width}
\sup_{\|v\|_{H^{s_2}(D_2)^{d_{out}}}\le 1}
\inf_{y\in (W_{N,K}^{(2)})^{d_{out}}}
\|v-y\|_{L^2(D_2)^{d_{out}}}
\lesssim
N^{-s_2/d_2}.
\end{equation}
Consequently,
\begin{equation}
\label{eq:constructive_input_width_scaled}
\sup_{u\in K}
\inf_{y\in (V_{N,K}^{(1)})^{d_{in}}}
\|u-y\|_{L^2(D_1)^{d_{in}}}
\lesssim
R_{K,\mathrm{in}}\,N^{-s_1/d_1},
\end{equation}
and
\begin{equation}
\label{eq:constructive_output_width_scaled}
\sup_{u\in K}
\inf_{y\in (W_{N,K}^{(2)})^{d_{out}}}
\|G^+(u)-y\|_{L^2(D_2)^{d_{out}}}
\lesssim
R_{K,\mathrm{out}}\,N^{-s_2/d_2}.
\end{equation}

Choose orthonormal bases $\{\varphi_n\}_{n=1}^{\infty}$ and $\{\psi_n\}_{n=1}^{\infty}$ of $L^2(D_1)$ and $L^2(D_2)$ so that
\[
\operatorname{span}\{\varphi_n\}_{n=1}^N=V_{N,K}^{(1)},
\qquad
\operatorname{span}\{\psi_n\}_{n=1}^N=W_{N,K}^{(2)},
\]
and the basis elements are piecewise polynomial.
\paragraph{ \textit{Step 2: Finite-dimensional encoding and decoding}}
For $N \in \mathbb{N}$, define the linear maps $F_{Nd_{in}} : L^2(D_1)^{d_{in}} \to \mathbb{R}^{Nd_{in}}$ and $F_{Nd_{out}} : L^2(D_2)^{d_{out}} \to \mathbb{R}^{Nd_{out}}$ by
\begin{align}\label{def:F_N}
F_{Nd_{in}}u \eqdef \left( (u, \varphi_1),..., (u, \varphi_N) \right) \in \mathbb{R}^{Nd_{in}}
,\hspace{1.0 em}
F_{Nd_{out}}v \eqdef \left( (v, \psi_1),..., (v, \psi_N) \right) \in \mathbb{R}^{Nd_{out}},
\end{align}
where $u=(u_1,...,u_{d_{in}}) \in L^2(D_1)^{d_{in}}$ and $
(u, \varphi_n)=
\left(
(u_1, \varphi_n)_{L^2(D_1)},...,(u_{d_{in}}, \varphi_n)_{L^2(D_1)}
\right) \in \mathbb{R}^{d_{in}}
$. Similarly, $v=(v_1,...,v_{d_{out}}) \in L^2(D_2)^{d_{out}}$ and
$
(v, \psi_n) = \left(
(v_1, \psi_n)_{L^2(D_2)},...,(v_{d_{out}}, \psi_n)_{L^2(D_2)}
\right) \in \mathbb{R}^{d_{out}}$.
We define $G_{Nd_{in}} : \mathbb{R}^{Nd_{in}} \to L^2(D_1)^{d_{in}}$ and $G_{Nd_{out}} : \mathbb{R}^{Nd_{out}} \to L^2(D_2)^{d_{out}} $ by
\begin{align}\label{def:G_N}
G_{Nd_{in}} \alpha \eqdef \sum_{n  \le  N} \alpha_n \varphi_n
,\hspace{1.0 em}
G_{Nd_{out}} \beta \eqdef \sum_{n  \le  N} \beta_n
\psi_n
,
\end{align}
where $
\alpha = (\alpha_1,...,\alpha_N) \in \mathbb{R}^{Nd_{in}}$, with $\alpha_n \in \mathbb{R}^{d_{in}}$, and $\beta = (\beta_1,...,\beta_N) \in \mathbb{R}^{Nd_{out}}$, $\beta_n \in \mathbb{R}^{d_{out}}.
$

By the $1$-bounded approximation property ($1$-BAP) we note that
\begin{equation}
\label{eq:MAP}
\|F_{Nd_{in}}\|_{\mathrm{op}} = \|F_{Nd_{out}}\|_{\mathrm{op}} = \|G_{Nd_{in}}\|_{\mathrm{op}} = \|G_{Nd_{out}}\|_{\mathrm{op}} = 1.
\end{equation}
\noindent Let $P_{V_N} : L^2(D_1) \to L^2(D_1)$ and $P_{W_N} : L^2(D_2) \to L^2(D_2)$ be an orthogonal projection onto $V_N$ and $W_N$ respectively; where
\[
\smash{
V_N \eqdef \mathrm{span}\{ \varphi_n \}_{n  \le  N}
,\hspace{1.0 em}
\mbox{ and }
\hspace{1.0 em}
W_N \eqdef \mathrm{span}\{ \psi_n \}_{n  \le  N}
.
}
\]
We denote by
\begin{align}
\label{def:C-K-N}
 C_{K}(N)  &\eqdef \sup_{a \in K} \|(I-P_{V_N})a\|_{L^2(D_1)^{d_{in}}} ,
\\
\label{def:C-GK-N}
    C_{G^{+}(K)}(N)
& \eqdef
    \sup_{a \in K} \|(I-P_{W_N})G^{+}(a)\|_{L^2(D_2)^{d_{out}}}
=
    \sup_{u \in G^+(K)} \|(I-P_{W_N}) u\|_{L^2(D_2)^{d_{out}}}
.
\end{align}
$C_K(N)$ and $C_{G^+(K)}(N)$ denote the constructive projection errors associated with the chosen componentwise subspaces $V_{N,K}^{(1)}$ and $W_{N,K}^{(2)}$.
We will bound $C_K(N)$ and $C_{G^+(K)}(N)$ later in the proof.
For $u \in L^2(D_1)^{d_{in}}$,
\[
G_{Nd_{in}}F_{Nd_{in}}u = \sum_{n  \le  N} (u, \varphi_n) \varphi_n=\left(P_{V_N}u_1, ..., P_{V_N}u_{d_{in}} \right),
\]
and $v \in L^2(D_2)^{d_{out}}$,
$
G_{Nd_{out}}F_{Nd_{out}}v = \sum_{n  \le  N} (v, \psi_n) {\psi}_n=\left(P_{{W}_N}v_1, ..., P_{{W}_N}v_{d_{out}} \right).
$
We now choose $N=N(\varepsilon, K, G^{+}) \in \mathbb{N}$ such that
\begin{equation}
\label{eq:C_estimates__earlyform}
        C_{G^{+}(K)}(N) + \omega(C_{K}(N))
    \le
        \varepsilon/2
.
\end{equation}
Define $
G_1\eqdef G_{Nd_{out}}F_{Nd_{out}} \circ G^{+} \circ G_{Nd_{in}} F_{Nd_{in}}.
$
Then, for $a \in K$, we estimate
\allowdisplaybreaks
\begin{align}\label{est:UAT1}
\| G^{+}(a) &- G_1(a) \|_{L^2(D_2)}
 \le
\| G^{+}(a) - G_{Nd_{out}}F_{Nd_{out}}G^{+}(a) \|
\nonumber
\\
& +
\| G_{Nd_{out}}F_{Nd_{out}}G^{+}(a) - G_{Nd_{out}}F_{Nd_{out}}G^{+} (G_{Nd_{in}}F_{Nd_{in}}a) \|
\nonumber
\\
&
 \le
C_{G^{+}(K)}(N)
+
\|G_{Nd_{out}}F_{Nd_{out}}\|_{\mathrm{op}} \omega(C_{K}(N))
 \le  C_{G^{+}(K)}(N)  + \omega(C_{K}(N)) \le
{\varepsilon} / 2.
\end{align}
\paragraph{\textit{Step 3: Representation by finite-rank operators}}
We denote by $
\psi = F_{Nd_{out}} \circ G^{+} \circ G_{Nd_{in}} : \mathbb{R}^{Nd_{in}} \to \mathbb{R}^{Nd_{out}}.
$
Since $\|F_{Nd_{out}}\|_{\mathrm{op}} = 1$ and $\| G_{Nd_{in}}\|_{\mathrm{op}} = 1$, the map
$\psi $ is uniformly continuous with a concave modulus of continuity $\omega$.

Choose the piecewise-polynomial bases above so that, after reordering and increasing $N$ by at most one if necessary, they include the constant functions $\varphi_1 = 1/|D_1|^{1/2}$ and $\psi_1 = 1/|D_2|^{1/2}$. This harmless rank change is absorbed in the stated order bounds.
Thus,
\[
F_{Nd_{in}}u(x) =
\sum_{m\leq N}
C^{(0)}_{m,1}(u, \varphi_m) \psi_1(x), \qquad x \in D_2,
\]
\noindent where $C^{(0)} \in \mathbb{R}^{Nd_{in} \times d_{in}}$ has block structure $C^{(0)}_{m,1}
=(
{\boldsymbol{\mathrm{O}}},...,{\boldsymbol{\mathrm{O}}}, \overbrace{\mathrm{diag}(|D_2|^{1/2})}^{m\text{th}}, {\boldsymbol{\mathrm{O}}},\dots,{\boldsymbol{\mathrm{O}}}
)^{\top}
$, with ${\boldsymbol{\mathrm{O}}}$ the zero matrix in $\mathbb{R}^{d_{in}\times d_{in}}$, and $\mathrm{diag}(|D_2|^{1/2})$ the diagonal matrix in $\mathbb{R}^{d_{in}\times d_{in}}$ whose diagonal entries are all $|D_2|^{1/2}$.

We now define the finite rank operator $K^{(0)}_N : L^2(D_1)^{d_{in}} \to L^2(D_2)^{Nd_{in}}$ by $
K^{(0)}_{N}u(x) \eqdef \sum_{m \le  N} C^{(0)}_{m,1}(u,\varphi_m) {\psi}_1(x)$, for $x \in D_2$.
Then,
\begin{equation*}
\resizebox{1\linewidth}{!}{$
\begin{aligned}
\left \Vert K^{(0)}_{N}u \right \Vert_{L^2(D_2)^{N d_{in}}}
=
\left\Vert \sum_{m \leq N} C^{(0)}_{m,1}(u,\varphi_m) \psi_1 \right \Vert
&=
\left(
\left \Vert
\sum_{m \leq N}
C^{(0)}_{m,1}(u,\varphi_m)
\right \Vert_{L^2(D_2)^{N d_{in}}}^{2}
\right)^{1/2}
\\
&=
|D_2|^{1/2}
\left(
\sum_{m \leq N}
\left \Vert (u,\varphi_m) \right \Vert_{2}^{2}
\right)^{1/2}
\leq
|D_2|^{1/2}
\|u\|_{L^2(D_1)^{d_{in}}}.
\end{aligned}
$}
\end{equation*}
Hence,
$
\|K^{(0)}_{N}\|_{\mathrm{op}} \leq
|D_2|^{1/2}.
$
% \end{equation}
%
For $u \in L^2(D_1)^{d_{in}}$, $F_{Nd_{in}}u = K^{(0)}_{N}u$.
Similarly, $G_{Nd_{out}} \beta
=
\sum_{n \leq N}\beta_{n}{\psi}_{n},
$ where $\beta = (\beta_1,...,\beta_N) \in \mathbb{R}^{Nd_{out}}$, $\beta_n \in \mathbb{R}^{d_{out}}$. As $
\beta_{n} = (1/|D_2|^{1/2}) (\beta_n, {\psi}_1),$ we have
$
G_{Nd_{out}} \beta =
\sum_{n  \le  N} C^{(L+1)}_{1,n}(\beta,{\psi}_1){\psi}_n(x)
$ where $C^{(L+1)}_{1,n} \in \mathbb{R}^{d_{out} \times Nd_{out}}$ is given by
$$
C^{(L+1)}_{1,n} =
(
{\boldsymbol{\mathrm{O}}},...,{\boldsymbol{\mathrm{O}}}, {\mathrm{diag}(1/|D_2|^{1/2})}, {\boldsymbol{\mathrm{O}}},...,{\boldsymbol{\mathrm{O}}}
)^\top,$$ where ${\boldsymbol{\mathrm{O}}}$ is the zero matrix in $\mathbb{R}^{d_{out}\times d_{out}}$, nonzero at the $n$th position,
and $\mathrm{diag}(1/|D_2|^{1/2})$ is the diagonal matrix in $\mathbb{R}^{d_{out}\times d_{out}}$  whose entries are all $1/|D_2|^{1/2}$.
We define the finite rank operator $K^{(L+1)}_{N} : L^2(D_2)^{Nd_{out}} \to L^2(D_2)^{d_{out}}$ by $K^{(L+1)}_{N}v(x)
=
\sum_{n \le  N} C^{(L+1)}_{1,n}(v, \psi_1) \psi_n(x)$,  for $x \in D_2$. Note that for $v \in L^2(D_2)^{Nd_{out}}$
\begin{align*}
\|K^{(L+1)}_{N}v\|_{L^2(D_2)^{d_{out}}}
&=
\| \sum_{n \leq N} C^{(L+1)}_{1,n}(v,\psi_1) \psi_n \|
=
\left(
\sum_{n \leq N}
\|
C^{(L+1)}_{1,n}(v,\psi_1)
\|_{2}^{2}
\right)^{1/2}
\\
&
=
(1/|D_2|^{1/2})
\left(
\sum_{n \leq N}
\|(v_n, \psi_1)\|_{2}^{2}
\right)^{1/2}
%\\
%&
\leq
(1/|D_2|^{1/2}) \|v\|_{L^{2}(D_2)^{Nd_{out}}}
\end{align*}
Hence, $
\|K^{(L+1)}_{N}\|_{\mathrm{op}} \leq (1/|D_2|^{1/2}).
$
By construction,
$K^{(L+1)}_{N} \beta = G_{Nd_{out}}\beta$ if $\beta \in \mathbb{R}^{Nd_{out}}$.
Therefore, the operator $G_1$ can be expressed as
\[
\smash{
G_1= G_{Nd_{out}}F_{Nd_{out}} \circ G^{+} \circ G_{Nd_{in}} F_{Nd_{in}}
= K^{(L+1)}_{N} \circ \psi \circ K^{(0)}_{N}.
}
\]
\paragraph{\textit{Step 4: Employing quantitative universal approximation for MLPs}}
We denote by $\mathcal{NN}_{\Delta,k :p,m}^{\tanh}$ the class of MLP with $p$ input neurons, $m$ output neurons, an arbitrary number of hidden layers of depth $\Delta$, at-most $k$ neurons per hidden layer, and $\tanh$ activation (\cref{def:MLPs}).
Since $K$ is compact and $K^{(0)}_{N}$ is a finite rank operator, the image $K^{(0)}_{N}(K) \subset \mathbb{R}^{Nd_{in}}$ is compact.
Also,
% \[
$\psi : \mathbb{R}^{Nd_{in}} \to \mathbb{R}^{Nd_{out}}, $
% \]
is a continuous map between Euclidean spaces.
By \cite[Proposition 53]{kratsios2022universal}, there exist a depth $L$ and a network
$\psi_{NN} \in \mathcal{NN}_{L, Nd_{in}+Nd_{out} + 2:Nd_{in},Nd_{out}}^{\tanh}$ such that
\[
\sup_{\tilde{a} \in K^{(0)}_{N}(K)}\|\psi(\tilde{a}) - \psi_{NN}(\tilde{a})\|
\le
\tfrac{\varepsilon}{2(1/|D_2|^{1/2})}.
\]
The depth $L$ of $\psi_{NN}$ can be chosen of order
\begin{equation*}
\resizebox{1\linewidth}{!}{$
\mathcal{O}\left(
Nd_{out} (\mathrm{diam}[K^{(0)}_{N}(K)])^{Nd_{in}}
\left(\omega^{-1}(F_{Nd_{out}} \circ G^{+} \circ G_{Nd_{in}}, \tfrac{{\varepsilon}}{(1+\tfrac{Nd_{in}}{4})2(1/|D_2|^{1/2})Nd_{out}} )
\right)^{-2Nd_{in}}
\right).
$}
\end{equation*}
Since $\|K^{(0)}_{N}\|_{\mathrm{op}} \leq |D_2|^{1/2}$,
\[
\mathrm{diam}[K^{(0)}_{N}(K)] \leq |D_2|^{1/2}\operatorname{diam}_{L^2(D_1)^{d_{in}}}(K).
\]
Thus, we can estimate
the depth of $\psi_{NN}$
from above by
\begin{equation}\label{estimate-layer}
\resizebox{1\linewidth}{!}{$
\mathcal{O}\left(
Nd_{out} \left(  |D_2|^{1/2}\operatorname{diam}_{L^2(D_1)^{d_{in}}}(K)
\right)^{Nd_{in}}
\left(\omega^{-1}(F_{Nd_{out}} \circ G^{+} \circ G_{Nd_{in}}, \tfrac{{\varepsilon}}{(1+\tfrac{Nd_{in}}{4})2(1/|D_2|^{1/2})Nd_{out}} )
\right)^{-2Nd_{in}}
\right).
$}
\end{equation}
Hence, for $a \in K$
\begin{align}\label{est:UAT3}
% &
\|G_1(a) - K^{(L+1)}_{N} \circ \psi_{NN} \circ K^{(0)}_{N}(a)\|
% \\
&
=
\|K^{(L+1)}_{N} \circ (\psi -  \psi_{NN}) \circ K^{(0)}_{N}(a) \|
% \nonumber
\\
\nonumber
&
\le
\|K^{(L+1)}_N\|_{\mathrm{op}}
\sup_{\tilde{a} \in K^{(0)}_{N}(K)}\|\psi(\tilde{a}) - \psi_{NN}(\tilde{a})\|
\\
&
\nonumber
\le
(1/|D_2|^{1/2})
\tfrac{{\varepsilon}}{2(1/|D_2|^{1/2})}
\nonumber
= \tfrac{{\varepsilon}}{2}.
\end{align}
By the metric approximation property~\eqref{eq:MAP}, the maps $F_{Nd_{out}}$ and $G_{Nd_{in}}$ are $1$-Lipschitz. Hence the modulus of continuity of $F_{Nd_{out}} \circ G^{+} \circ G_{Nd_{in}}$ is bounded above by $\omega$, and~\eqref{estimate-layer} reduces to
\begin{equation}\label{estimate-layer__II}
\mathcal{O}\left(
Nd_{out} \left(  |D_2|^{1/2}\operatorname{diam}_{L^2(D_1)^{d_{in}}}(K)
\right)^{Nd_{in}}
\left(\omega^{-1}
    \biggl(
        \tfrac{{\varepsilon}}{(1+\tfrac{Nd_{in}}{4})2(1/|D_2|^{1/2})Nd_{out}}
    \biggr)
\right)^{-2Nd_{in}}
\right).
\end{equation}
\paragraph{\textit{Step 5: Putting it all together}}
Define $
G=K^{(L+1)}_{N} \circ \psi_{NN} \circ K^{(0)}_{N}$
which serves as an approximator of $G^{+}$.
By construction, with $W=Nd_{in}+Nd_{out}+2$ and zero coordinate-bias networks of depth at most one,
$
G \in \mathcal{NO}^{\tanh}_{N,W,L,1,\boldsymbol d}
$
with depth $L$ bounded as in \eqref{estimate-layer__II}.
Combining estimates (\ref{est:UAT1}) and (\ref{est:UAT3}), we have for $a \in K$
\[
\begin{aligned}
    \| G^{+}(a) - G(a) \|
\le
    \| G^{+}(a) - G_{1}(a) \|
    +
    \| G_{1}(a) - G(a) \|
\le
    \tfrac{{\varepsilon}}{2}+\tfrac{{\varepsilon}}{2}
=
    {\varepsilon}.
\end{aligned}
\]
It remains to make the choice of $N$ explicit. Since the scalar subspaces $V_{N,K}^{(1)}$ and $W_{N,K}^{(2)}$ were chosen in Step~1 and applied componentwise,
\begin{equation}
\label{eq:final_ck_bound}
C_K(N)
\lesssim
R_{K,\mathrm{in}}\,N^{-s_1/d_1},
\qquad
C_{G^{+}(K)}(N)
\lesssim
R_{K,\mathrm{out}}\,N^{-s_2/d_2}.
\end{equation}
Therefore, a sufficient condition for~\eqref{eq:C_estimates__earlyform} is
\[
R_{K,\mathrm{out}}\,N^{-s_2/d_2}\le \varepsilon/4
\qquad\text{and}\qquad
R_{K,\mathrm{in}}\,N^{-s_1/d_1}\le \omega^{-1}(\varepsilon/4).
\]
Equivalently, it suffices to take
\[
N
\in
\mathcal{O}\!\left(
\left(\frac{R_{K,\mathrm{in}}}{\omega^{-1}(\varepsilon)}\right)^{d_1/s_1}
\vee
\left(\frac{R_{K,\mathrm{out}}}{\varepsilon}\right)^{d_2/s_2}
\right).
\]
With this choice, the previous construction yields the desired neural operator approximation on $K$.

\subsection[Proof of the routing-tree lemma]{Proof of Lemma~\ref{lem:complete_n_ary_tree_counting}} \label{s:Proof__ss:TreeLemma}
Fix $\delta>0$, and consider a positive integer $k$, to be set later.
 \paragraph{    \textit{Step 1: Optimal covering of $K$ via metric entropy}}
We begin with a quantitative version of the Rellich-Kondrashov Theorem for $K$; for $k\in \mathbb{N}$, consider the $k$-entropy number of $K$,
\[
        e_k(K)
    \eqdef
        \inf
        \big\{\delta>0:\,
            \exists f_1,\dots,f_{2^k-1}\in L^2(D_1)
            \,
            \max_{f\in K}\,
                \min_{i=1,\dots,2^k-1}\,
                \|
                    f-f_i
                \|_{L^2(D_1)}
            <
                \delta
        \big\}
.
\]
In \cite{birman1967piecewise}, it was shown that $e_k(K)$ satisfies
\[
        k^{-s_1/d_1}
    \lesssim
        e_{k}(K)
    \lesssim
        k^{-s_1/d_1}.
\]
Thus, after increasing the constant if necessary, there exist $2^k-1$ functions $f_1,\dots,f_{2^k-1}\in K$ such that
\begin{equation}
\label{eq:covering_raw}
        \max_{f\in K}\, \min_{i=1,\dots, 2^k-1}\, \|f-f_i\|_{L^2(D_1)}
    <
        C\,
            k^{-s_1/d_1}
\end{equation}
for some absolute constant $C>0$.
\paragraph{ \textit{Step 2: Building an idealized $v$-ary tree with leaves containing $\{f_i\}_{i=1}^{2^k-1}$}}
We recall that a complete $v$-ary tree of height $h$ has leaves $\Lambda$, and total vertices (nodes) $V$ given by
$
        \Lambda = v^h
    \mbox{ and }
       V = \frac{v^{h+1}-1}{v-1}
.
$
We take $h$ to be the least integer for which $\Lambda=v^h\ge 2^k-1$.
It yields
\allowdisplaybreaks
\begin{align}
\label{eq:identifying_height__UB}
        h
    =
        \big\lceil
            \log_v\big(
                  2^k -1
                \big)
        \big\rceil
    =
        \big\lceil
            \log_v\big(
                  2^{\lceil \delta^{-d_1/s_1}\rceil} -1
                \big)
        \big\rceil
    ,
\end{align}
where the ceiling ensures that $h$ is an integer.

We now construct an \emph{idealized} $v$-ary tree defining the decentralized NO using backwards recursion. For each $i=1,\dots,2^k-1$, relabel $f_{h:i}\eqdef f_i$, and for the remaining indices $i=2^k,\dots,v^h$ set $f_{h:i}\eqdef f_1$. Define $\tilde{V}_h\eqdef \{f_{h:i}\}_{i=1}^{v^h}$ and $\tilde{E}_h\eqdef \emptyset$. For $\tilde{h}=h-1,h-2,\dots,0$, define the nonlinear functional $\ell_{\tilde{h}}: \, L^2(D_1)^{v^{\tilde{h}}}  \rightarrow [0,\infty)$ by
\begin{equation}
\label{eq:NN_Energy}
\smash{
    (f_1,\dots,f_{v^{\tilde{h}}})  \xrightarrow{\ell_{\tilde{h}}}\sum_{k=1}^{v^{\tilde{h}}}\, \sum_{j=1}^{v^{\tilde{h}+1}}\, \|f_k-f_{\tilde{h}+1:j}\|_{L^2(D_1)}.
}
\end{equation}
Since the norm $\|\cdot\|_{L^2(D_1)}$ is continuous in $L^2(D_1)$, $\ell_{\tilde{h}}$ is continuous on $L^2(D_1)^{v^{\tilde{h}}}$. As $K$ is compact in $L^2(D_1)$, $K^{v^{\tilde{h}}}$ is also compact in $L^2(D_1)^{v^{\tilde{h}}}$ with its product topology. Therefore, there exists a minimizer of $\ell_{\tilde{h}}$ on $K^{v^{\tilde{h}}}$; pick one such minimizing family $\{f_{\tilde{h}:k}\}_{k=1}^{v^{\tilde{h}}}$.
Define
\[
\smash{
    \tilde{V}_{\tilde{h}} \eqdef \{f_{\tilde{h}:k}\}_{k=1}^{v^{\tilde{h}}} \cup \tilde{V}_{\tilde{h}+1}
}
\]
For each $j=1,\dots,v^{\tilde{h}+1}$, choose an index (possibly not unique) $i^{\star:j}\in \{1,\dots,v^{\tilde{h}}\}$ such that
\[
\smash{
        \|f_{\tilde{h}:i^{\star:j}}-f_{\tilde{h}+1:j}\|_{L^2(D_1)}
    =
    \min_{i=1,\dots,v^{\tilde{h}}}
        \,
        \|f_{\tilde{h}:i}-f_{\tilde{h}+1:j}\|_{L^2(D_1)}
.
}
\]
Define the updated edge set as
$\tilde{E}_{\tilde{h}}
    \eqdef
            \Big\{
                \{
                    f_{\tilde{h}:i^{\star:j}}
                    ,
                    f_{\tilde{h}+1:j}
                :
                \,
                j=1,\dots,v^{\tilde{h}+1}
                \}
            \Big\}
        \bigcup
            \tilde{E}_{\tilde{h}+1}$.
% Final Tree
Finally, define the complete idealized tree as $\tilde{\mathcal{T}}\eqdef (\tilde{V}_0,\tilde{E}_0)$.

\paragraph{\textit{Step 3: Implementing the nodes and edges in the idealized trees using MLPs}}
Retroactively set $k\eqdef \lceil \delta^{-d_1/s_1}\rceil$. Since $s_1>d_1$ by assumption, Sobolev embedding gives continuous representatives for the finitely many centers in $\tilde V_0$. ReLU MLPs are dense in $C([0,1]^{d_1})$, and the router proof does not require a size bound for these MLPs. Thus, after enlarging the constant $C$ from~\eqref{eq:covering_raw} if needed, for each center $f\in\tilde V_0$ we may choose a ReLU MLP $\hat f$ satisfying
\begin{equation}
\label{eq:quantization}
        \max_{f\in \tilde V_0}\,
            \|f-\hat{f}\|_{L^{\infty}( [0,1]^{d_1} )}
    \le
        C\, \delta
\end{equation}.
Since $\|\cdot\|_{L^2([0,1]^{d_1})}\le \|\cdot\|_{L^{\infty}([0,1]^{d_1})}$, then~\eqref{eq:covering_raw} and~\eqref{eq:quantization} imply that
\begin{equation*}
\begin{aligned}
        \max_{f\in K}\, \min_{i=1,\dots,
        {2^{\lceil \delta^{-d_1/s_1}\rceil} -1 }
        }\, \|f-\hat{f}_i\|_{L^2( [0,1]^{d_1} )}
    \le &
        \max_{f\in K}\, \min_{i=1,\dots,
        {2^{\lceil \delta^{-d_1/s_1}\rceil} -1 }
        }\,
        \|f-f_i\|_{L^2( [0,1]^{d_1} )}
\\
+
&
\|f_i-\hat{f}_i\|_{L^2( [0,1]^{d_1} )}
    \le
        C\,k^{-s_1/{d_1}}
\\
+
&
        \min_{i=1,\dots,
        % {2^k-1}
        {2^{\lceil \delta^{-d_1/s_1}\rceil} -1 }
        }\,
        \|f_i-\hat{f}_i\|_{L^2( [0,1]^{d_1} )}
\\
	    <
	&
	        2\,
	        C
	        \, k^{-s_1/{d_1}}
	    \le
	    C_R\, \delta
	,
\end{aligned}
\end{equation*}
where we have defined $C_R \eqdef 2C$.
We update the idealized tree $\tilde{\mathcal{T}}$ by defining $\tilde{\mathcal{T}}\eqdef (V,E)$ where
$
        V
    \eqdef
        \{
	            \hat{f}:\,f\in \tilde V_0
        \}
\mbox{ and }
        E
    \eqdef
        \{
	            \{\hat{f}_i,\hat{f}_j\}:\, (f_i,f_j)\in \tilde E_0
        \}
.
$
\paragraph{\textit{Step 4: Verifying the approximate $K$-means relations}}
From~\eqref{eq:quantization}, for each $\tilde{h}\in \{0,\dots,h-1\}$, set
\[
S_{\tilde h}\eqdef
\sum_{k=1}^{v^{\tilde{h}}}
\sum_{j=1}^{v^{\tilde{h}+1}}
\|\hat{f}_{\tilde{h}:k}-\hat{f}_{\tilde{h}+1:j}\|_{L^2(D_1)}.
\]
Then
\allowdisplaybreaks
\begin{align*}
S_{\tilde h}
\le &
    \sum_{k=1}^{v^{\tilde{h}}}\,
    \sum_{j=1}^{v^{\tilde{h}+1}}\,
        \Big(
            \|\hat{f}_{\tilde{h}:k}-f_{\tilde{h}:k}\|_{L^2(D_1)}
        +
            \|f_{\tilde{h}:k}-f_{\tilde{h}+1:j}\|_{L^2(D_1)}
        +
            \|f_{\tilde{h}+1:j}-\hat{f}_{\tilde{h}+1:j}\|_{L^2(D_1)}
        \Big)
\\
\nonumber
\le &
    \sum_{k=1}^{v^{\tilde{h}}}\,
    \sum_{j=1}^{v^{\tilde{h}+1}}\,
        \Big(
            C_R\delta
        +
            \|f_{\tilde{h}:k}-f_{\tilde{h}+1:j}\|_{L^2(D_1)}
        +
            C_R\delta
        \Big)
=
    2C_R\delta
    \,
    v^{2\tilde{h}+1}
\\
\nonumber
&
+
    \sum_{k=1}^{v^{\tilde{h}}}\,
    \sum_{j=1}^{v^{\tilde{h}+1}}\,
        \|f_{\tilde{h}:k}-f_{\tilde{h}+1:j}\|_{L^2(D_1)}
\\
\le
&\quad
2C_R\delta\,v^{2\tilde{h}+1}
+
\min_{\substack{f_1,\dots,f_{v^{\tilde{h}}}\in K}}
\sum_{k=1}^{v^{\tilde{h}}}
\sum_{j=1}^{v^{\tilde{h}+1}}
\|f_{\tilde{h}:k}-f_{\tilde{h}+1:j}\|_{L^2(D_1)}
\end{align*}
where we have used the definition of $\{f_{\tilde{h}:k}\}_{k=1}^{v^{\tilde{h}}}$ as a minimizing family for the nearest-neighbor energy functional $\ell_{\tilde{h}}$, as defined in~\eqref{eq:NN_Energy}. This completes our proof.

\subsection[Proof of the domain-extension proposition]{Proof of Proposition~\ref{prop:extension_lemma}}\label{proof:extension_lemma}
Since $H^{s_1}(D_1)$ is dense in $L^2(D_1)$ and $L^2(D_2)$ is complete, the uniformly continuous map $G^+$ has a unique continuous extension $\tilde G:L^2(D_1)\to L^2(D_2)$ with the same modulus of continuity. Indeed, for $u\in L^2(D_1)$ choose $u_n\in H^{s_1}(D_1)$ with $u_n\to u$ in $L^2(D_1)$ and set $\tilde G(u)=\lim_n G^+(u_n)$; uniform continuity makes this limit independent of the approximating sequence and preserves the modulus.
For $i=1,2$, consider the \textit{restriction} $\rho_{(i)}: L^2([0,1]^{d_i})\to L^2(D_i)$ and \textit{extension by zero} $E^{(i)}:L^2(D_i)\to L^2([0,1]^{d_i})$ operators defined by
\begin{equation}
\begin{aligned}
\rho_{(i)}(u) & \eqdef u|_{D_i}
\quad\mbox{ and }\quad
E^{(i)}(u)  \eqdef u \boldsymbol{1}_{D_i}.
\end{aligned}
\end{equation}
Define
\[
\bar G \eqdef E^{(2)}\circ \tilde G\circ \rho_{(1)}:
L^2([0,1]^{d_1})\to L^2([0,1]^{d_2}).
\]
Then $\bar{G}$ is uniformly continuous with modulus of continuity $\omega$. Moreover, by construction, for each $u\in H^{s_1}(D_1)$
\allowdisplaybreaks
\begin{align}
\label{eq:extension_yoga2}
\bar{G}(\bar{u})|_{D_2}
\eqdef
\bar{G}(E^{(1)}(u))|_{D_2}
=
\rho_{(2)}\circ E^{(2)}\circ \tilde{G}\circ \rho_{(1)}\circ E^{(1)}(u)
=
\tilde{G}(u)
=
G^+(u).
\end{align}

The right-hand side of \eqref{eq:extension_yoga2} holds since $\tilde{G}$ extends $G^+$, and the left-hand side uses $E^{(1)}(u)=\bar{u}$.

Now, observe that, by construction, for $i=1,2$, $\rho_{(i)}\circ E^{(i)}$ is the identity on $L^2(D_i)$.
Moreover, both maps are $1$-Lipschitz since: for each $i=1,2$, every $u\in L^2(D_i)$, and each $v\in L^2([0,1]^{d_i})$
\allowdisplaybreaks
\begin{align*}
\|E^{(i)}(u)\|_{L^2([0,1]^{d_i})}^2
&=
\int_{x\in D_i}\|u(x)\|^2 dx
+
\int_{x\in [0,1]^{d_i}\setminus D_i} 0\, dx
\\
&=
\int_{x\in D_i}\|u(x)\|^2 dx
=
\|u\|_{L^2(D_i)}^2
\\
\|\rho_{(i)}(v)\|_{L^2(D_i)}^2
&=
\int_{x\in D_i} \|v(x)\|^2 dx
\\
&\le
\int_{x\in D_i} \|v(x)\|^2 dx
+
\int_{x\in [0,1]^{d_i}\setminus D_i} \|v(x)\|^2 dx
\\
&=
\|v\|_{L^2([0,1]^{d_i})}^2
.
\end{align*}

\subsection[Proof of the main result]{Proof of Theorem~\ref{thrm:Main}} \label{s:Proofs__ss:MainResult}
\paragraph{\textit{Step 1: Choice of the expert scale and of the tree resolution}}
Let
\[
R_{G,K}
\eqdef
\sup_{u\in K}\|G^+(u)\|_{H^{s_2}(D_2)}
<
\infty.
\]
For every local subset of $K$, Proposition~\ref{prop:quantitative_UAT__ON} can therefore be applied with input radius bounded by $R$ and output radius bounded by $R_{G,K}$. Fix
\[
N
\in
\mathcal{O}\!\left(
[\omega^{-1}(\varepsilon)]^{-d_1/s_1}
\vee
\varepsilon^{-d_2/s_2}
\right),
\]
with hidden constants depending on $R$ and $R_{G,K}$.

Let $C_R>0$ be the constant of Lemma~\ref{lem:complete_n_ary_tree_counting}, and define
\begin{equation}
\label{eq:R_Def__TReeMode}
\rho_\varepsilon
\eqdef
\omega^{-1}\!\left(
\tfrac{\varepsilon\,|D_2|^{1/2}}{(2+N\,d_{in}/2)N d_{out}}
\right),
\qquad
\delta_\varepsilon
\eqdef
\tfrac{\rho_\varepsilon^2}{4C_R}.
\end{equation}
By Lemma~\ref{lem:complete_n_ary_tree_counting}, there exists a routing tree $\mathcal T$ satisfying the approximate hierarchical relation~\eqref{eq:K_Means__Recursive} and the covering estimate
\begin{equation}
\label{eq:L2_diam_bound}
\max_{x\in K}\,
\min_{i=1,\dots,\Lambda}
\|x-\hat f_i\|_{L^2([0,1]^{d_1})}
<
C_R\delta_\varepsilon
=
\rho_\varepsilon^2/4.
\end{equation}
Its height is
\[
h
=
\left\lceil
\log_v\!\bigl(2^{\lceil \delta_\varepsilon^{-d_1/s_1}\rceil}-1\bigr)
\right\rceil,
\]
so the expert count is $\Lambda=v^h$.

\paragraph{\textit{Step 2: Local neural-operator approximation}}
For each leaf $i=1,\dots,\Lambda$, define the local subset
\[
K_i
\eqdef
\left\{
u\in K:\,
\|u-\hat f_i\|_{L^2(D_1)}\le \rho_\varepsilon^2/4
\right\}.
\]
By~\eqref{eq:L2_diam_bound}, every $u\in K$ belongs to at least one such set, and each $K_i$ has
\[
\operatorname{diam}_{L^2(D_1)}(K_i)\le \rho_\varepsilon^2/2.
\]
Applying Proposition~\ref{prop:quantitative_UAT__ON} to each $K_i$ yields an expert $\hat G_i$ satisfying
\[
\sup_{u\in K_i}\|G^+(u)-\hat G_i(u)\|_{L^2(D_2)}<\varepsilon.
\]
Moreover, each expert can be chosen with
\[
W_i
\in
\mathcal{O}(N),
\qquad
\operatorname{rank}(\hat G_i)
\in
\mathcal{O}(N),
\]
and, by the depth estimate in \Cref{prop:quantitative_UAT__ON},
\[
L_i
\in
\mathcal{O}\!\left(
N d_{out}
\left(
\frac{|D_2|^{1/2}\operatorname{diam}_{L^2(D_1)}(K_i)}{\rho_\varepsilon^2}
\right)^{N d_{in}}
\right)
\subseteq
\mathcal{O}(N),
\]
because $|D_2|=1$ and $\operatorname{diam}_{L^2(D_1)}(K_i)\le \rho_\varepsilon^2/2$.

\paragraph{\textit{Step 3: Active complexity}}
Let $\mathcal{NO}=\{\hat G_i\}_{i=1}^\Lambda$. The realization of $(\mathcal T,\mathcal{NO})$ uses the nearest-leaf rule in \Cref{alg:realization_mix}. Hence, by the covering estimate~\eqref{eq:L2_diam_bound}, each $u\in K$ is routed to a leaf $i$ with $u\in K_i$ and then evaluates $\hat G_i(u)$. The resulting MoNO approximates $G^+$ uniformly on $K$ with error at most $\varepsilon$.

Since each expert has width, rank, and depth of order $N$, while the bias networks can be taken to be zero in the construction of Proposition~\ref{prop:quantitative_UAT__ON}, the active parameter count satisfies
\[
\operatorname{Activ-Cpl}(\mathcal T,\mathcal{NO})
=
\max_{i=1,\dots,\Lambda} P(\hat G_i)
\in
\mathcal{O}(N^4).
\]
This proves the theorem.

\section{Additional Material}
\label{app:additional-material}

This appendix collects supplemental material that would otherwise interrupt the main narrative: the routing pseudocode used by MoNOs, the auxiliary domain-extension result, additional controlled-experiment diagnostics, an inverse-problem case study that motivates the distributed construction, and the explicit identification of a classical neural operator as a one-expert MoNO.

\paragraph{Routing pseudocode.}
For reference, the routing rule used in the main text can be written as the following nearest-leaf search over tree-indexed centers.

\begin{algorithm}[H]
\caption{Realization of a mixture of neural operators}
\label{alg:realization_mix}
\begin{algorithmic}[1]
\STATE \textbf{Input:} $u \in H^{s_1}(D_1)\subset L^2(D_1)$
\STATE Evaluate the $\Lambda$ leaf-center distances $\|\hat f_{h:i}-u\|_{L^2(D_1)}$
\STATE $\displaystyle i^\star=\argmin_{i\in\{1,\dots,\Lambda\}}\|\hat f_{h:i}-u\|_{L^2(D_1)}$
\STATE Follow the unique root-to-leaf path in $\mathcal T$ ending at leaf $i^\star$
\STATE \textbf{Output:} $G_{h,i^\star}(u)$
\end{algorithmic}
\end{algorithm}

\section{Routing Tree Lemma}
\label{app:routing-tree-lemma}

\begin{lemma}[Approximate hierarchical $K$-means in function space]
\label{lem:complete_n_ary_tree_counting}
Fix $s_1,R>0$ and $d_1\in\N$ with $s_1>d_1$. Let $K$ be the compact set in \eqref{eq:compact_set_Sobolev_type}. Then there exists a constant $C_R>0$ such that, for every valency $v\in\N_{\ge 2}$ and radius $\delta>0$, there is a $v$-ary tree $\mathcal T=(V,E)$ of height $h$ whose nodes are ReLU MLPs $\{\hat f\}_{\hat f\in V}$ and whose leaves satisfy
\begin{equation}
\label{eq:packing_condition}
\max_{x\in K}
\min_{i=1,\dots,\Lambda}
\|x-\hat f_i\|_{L^2([0,1]^{d_1})}
<
C_R\delta.
\end{equation}
Moreover, for each $\tilde h=0,\dots,h-1$, the tree obeys the approximate hierarchical $k$-means condition
\begin{equation}
\label{eq:K_Means__Recursive}
\sum_{k=1}^{v^{\tilde h}}
\sum_{j=1}^{v^{\tilde h+1}}
\|\hat f_{\tilde h:k}-\hat f_{\tilde h+1:j}\|_{L^2(D_1)}
\le
2C_R\delta\,v^{2\tilde h+1}
+
\min_{f_1,\dots,f_{v^{\tilde h}}\in K}
\sum_{k=1}^{v^{\tilde h}}
\sum_{j=1}^{v^{\tilde h+1}}
\|f_{\tilde h:k}-f_{\tilde h+1:j}\|_{L^2(D_1)}.
\end{equation}
Its leaf count, height, and total number of nodes are given by the bounds recorded in \Cref{s:Proof__ss:TreeLemma}.
\end{lemma}

\section{Domain Extension Auxiliary Result}
\label{app:domain-extension}

\begin{proposition}[Domain extension]
\label{prop:extension_lemma}
Assume the setting of \Cref{def:neural-operator-v2} with $d_{in}=d_{out}=1$. Let
\[
G^+:\bigl(H^{s_1}(D_1),\|\cdot\|_{L^2(D_1)}\bigr)
\to
\bigl(H^{s_2}(D_2),\|\cdot\|_{L^2(D_2)}\bigr)
\]
be uniformly continuous with modulus $\omega$, and suppose each $D_i$ is a compact subset of $[0,1]^{d_i}$ of positive Lebesgue measure. Then there exists a uniformly continuous operator
\[
\bar G:L^2([0,1]^{d_1})\to L^2([0,1]^{d_2})
\]
with the same modulus of continuity such that $\bar G(\bar u)|_{D_2}=G^+(u)$ whenever $\bar u|_{D_1}=u$.
\end{proposition}

\section{Additional Experimental Diagnostics}
\label{app:experiment_diagnostics}

This appendix records the protocol and exact parameter accounting for the Burgers experiments in \Cref{sec:controlled_experiment}, together with the older soft-mixture mechanism check used during development. These diagnostics support the controlled empirical checks in the main text.

\paragraph{Burgers protocol.}
For each seed, we generate $6144$ periodic initial conditions on a grid of size $64$ and split them into $4096/1024/1024$ train/validation/test samples. The reported numbers use seeds $7,17,27$. The route label $r(u_0)\in\{0,1,2,3\}$ is the empirical quartile of $\|\partial_xu_0\|_\infty$ in the generated input pool. In the homogeneous experiment, every sample is evolved with viscosity $\nu=0.01$; the same quartile labels are kept only to test whether a routed model helps when the PDE operator itself is not regime-dependent. In the heterogeneous experiment, samples in the four quartiles are evolved with viscosities $(0.05,0.02,0.008,0.003)$. All targets are produced by the explicit finite-volume Burgers solver used in the released script. Inputs and outputs are normalized using statistics from the training split. All models are trained for $40$ epochs with Adam, learning rate $10^{-3}$, batch size $128$, and validation-based checkpoint selection on a single NVIDIA RTX PRO 6000 Blackwell GPU.

\paragraph{Model rows.}
The active-matched global NO is a single Fourier-style neural operator with width $16$, $10$ Fourier modes, and $4$ spectral blocks. The total-matched global NO uses the same architecture with width $33$, chosen to match the stored parameter count of the learned-router MoNO. The oracle MoNO consists of four width-$16$ experts and no learned router; both training and evaluation dispatch sample $u_0$ to expert $r(u_0)$. The learned-router MoNO uses the same four experts plus a width-$64$ feature router. During training, the experts are still updated under the oracle assignment $r(u_0)$, while the router is trained with cross-entropy to predict $r(u_0)$. Its feature vector contains the input mean, standard deviation, root mean square, maximum, minimum, mean absolute periodic finite difference, maximum absolute periodic finite difference, and the first ten nonzero real-FFT magnitudes of the normalized input. At test time, the learned router selects $\argmax_i q_\theta(i\mid u_0)$ and only the selected expert is evaluated. Thus the learned-router row is a supervised-routing diagnostic. Active-including-router counts one expert plus the router; total stored parameters count all experts plus the router.

\begin{table*}[t]
\ra{1.15}
\centering
\begin{adjustbox}{max width=0.9\textwidth,center}
\begin{tabular}{@{}llrrrr@{}}
\cmidrule[0.3ex](){1-6}
\textbf{Case} & \textbf{Model} & \textbf{Active expert} & \textbf{Router} & \textbf{Total stored} & \textbf{Rel.\ $L^2$} \\
\midrule
Homogeneous & Global NO (active-matched) & $21{,}905$ & $0$ & $21{,}905$ & $0.0203 \pm 0.0004$ \\
 & Global NO (total-matched) & $92{,}863$ & $0$ & $92{,}863$ & $0.0119 \pm 0.0006$ \\
 & Oracle MoNO & $21{,}905$ & $0$ & $87{,}620$ & $0.0223 \pm 0.0007$ \\
 & Learned-router MoNO & $21{,}905$ & $5{,}572$ & $93{,}192$ & $0.0213 \pm 0.0013$ \\
\midrule
Heterogeneous & Global NO (active-matched) & $21{,}905$ & $0$ & $21{,}905$ & $0.1140 \pm 0.0056$ \\
 & Global NO (total-matched) & $92{,}863$ & $0$ & $92{,}863$ & $0.0986 \pm 0.0019$ \\
 & Oracle MoNO & $21{,}905$ & $0$ & $87{,}620$ & $0.0245 \pm 0.0007$ \\
 & Learned-router MoNO & $21{,}905$ & $5{,}572$ & $93{,}192$ & $0.0558 \pm 0.0033$ \\
\cmidrule[0.3ex](){1-6}
\end{tabular}
\end{adjustbox}
\caption{Exact parameter accounting and relative $L^2$ test error for the two Burgers cases. The homogeneous case uses fixed viscosity $\nu=0.01$. The heterogeneous case assigns viscosities $(0.05,0.02,0.008,0.003)$ by initial-condition steepness quartile. The learned-router rows report a larger router than the soft-mixture diagnostic because this router uses input statistics, spectral magnitudes, and gradient features.}
\label{tab:burgers_active_complexity}
\end{table*}

\paragraph{Soft-mixture diagnostic.}
The following table and figure record the earlier controlled soft-mixture operator-learning diagnostic. Its target is itself a soft mixture, so it is best read as a mechanism check; the Burgers experiment in the main text is the less synthetic empirical illustration.

\begin{table*}[t]
\ra{1.15}
\centering
\begin{adjustbox}{max width=0.9\textwidth,center}
\begin{tabular}{@{}lrrrrr@{}}
\cmidrule[0.3ex](){1-6}
\textbf{Model} & \textbf{Active expert} & \textbf{Router} & \textbf{Active incl.\ router} & \textbf{Total stored} & \textbf{Test MSE} \\
\midrule
Global NO (active-matched) & $21{,}905$ & $0$ & $21{,}905$ & $21{,}905$ & $0.000785 \pm 0.000102$ \\
Global NO (total-matched) & $87{,}329$ & $0$ & $87{,}329$ & $87{,}329$ & $0.000456 \pm 0.000020$ \\
Oracle MoNO & $21{,}905$ & $0$ & $21{,}905$ & $87{,}620$ & $0.000224 \pm 0.000004$ \\
Learned-router MoNO & $21{,}905$ & $2{,}212$ & $24{,}117$ & $89{,}832$ & $0.000234 \pm 0.000004$ \\
\cmidrule[0.3ex](){1-6}
\end{tabular}
\end{adjustbox}
\caption{Exact parameter accounting and test error for the controlled neural-operator illustration. The learned-router row reports both the active expert size and the slightly larger active count obtained when the router itself is included. The qualitative message is the same as in the theory: better approximation with a smaller active model can come at the price of a larger stored ensemble.}
\label{tab:soft_mixture_active_complexity}
\end{table*}

\begin{figure*}[t]
\centering
\includegraphics[width=\textwidth]{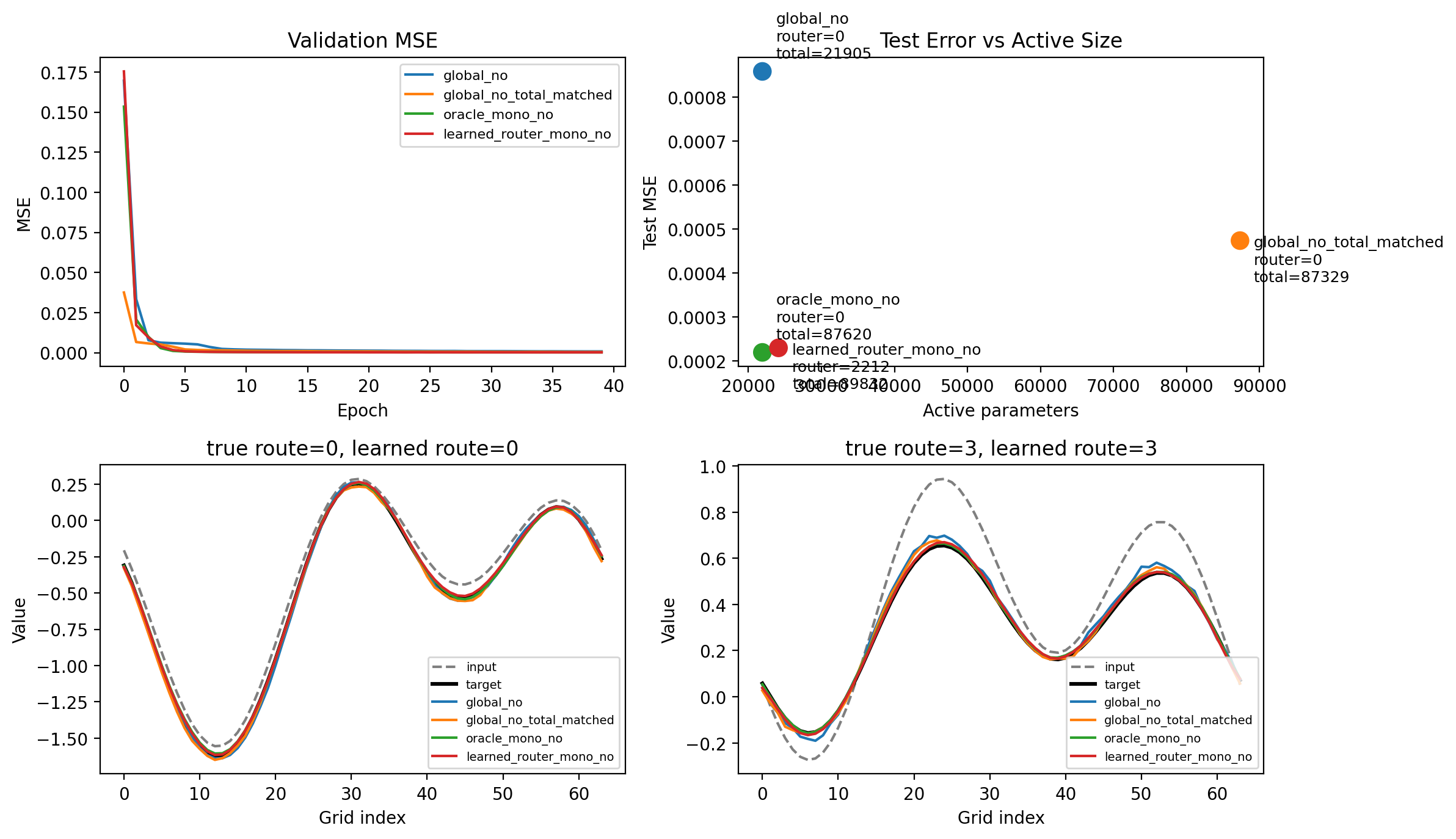}
\caption{\textbf{Representative diagnostics for the controlled neural-operator illustration.} Top left: validation MSE across epochs for one run with seed $7$ at the same $40$-epoch training budget used in the main text. Top right: the same active-versus-total comparison used in the main text, now shown for this single run. Bottom row: two representative test examples, with the input function shown as a dashed black curve and the target operator output shown in solid black. The routed MoNO predictions track the target more closely than the active-matched global neural operator on these examples, while the learned router selects the same route as the oracle on the displayed samples.}
\label{fig:soft_mixture_reconstructions_appendix}
\end{figure*}

\paragraph{Training-budget note.}
For the soft-mixture diagnostic, we also reran the same three-seed experiment for $100$ epochs, keeping the architecture, synthetic operator, and parameter accounting fixed. In that longer-budget rerun, the active-matched global NO attains test MSE $0.000360\pm0.000024$, the total-matched global NO attains $0.000158\pm0.000023$, the oracle MoNO attains $0.000191\pm0.000006$, and the learned-router MoNO attains $0.000179\pm0.000004$ with routing accuracy $0.993\pm0.001$. Thus, the same fixed-active-size conclusion remains visible under longer training: routed models still outperform the active-matched global baseline. At the same time, a larger-active total-matched global baseline can improve enough to outperform the routed models, which is fully consistent with the paper's claim, since the theory concerns active complexity rather than universal empirical dominance over larger-active models.

\section{Applications to Inverse Problems}\label{a:Inverseproblems}

We include an inverse-problem case study illustrating why localization can matter when the target operator has only weak stability.
% \section{Application: Feasible distributed approximation of an inverse operator for the boundary values problem associated with the Helmholtz equation}
% \label{s:Example}
%Cheng, Chouli, Lin: STABLE DETERMINATION OF A BOUNDARY COEFFICIENT IN AN ELLIPTIC EQUATION

Let $\Omega\subset \mathbb R^2$ be an open bounded connected set with smooth boundary $\partial \Omega$.
Consider the inverse problem of corrosion detection in electrostatics.
We model the conductor by the domain $\Omega$ and assume that it has an inaccessible
part of the boundary, denoted by $\Gamma\subset \partial \Omega$, that is affected by corrosion.
We model this by the equation
\begin{eqnarray} \label{ES problem1}
 & &\Delta u(x)=0,\quad\ \hbox{in }x\in \Omega,\\ \label{ES problem2}
& &\p_\nu u+qu=0,\quad\ \hbox{on }x\in \Gamma,\\ \label{ES problem3}
& &\p_\nu u=f,\quad\ \hbox{on }x\in\partial \Omega\setminus \Gamma.
\end{eqnarray} 
We denote the solution of \eqref{ES problem1}-\eqref{ES problem3} by $u^q=u$.
Here, $u=u(x)$, $u\in H^1(\Omega)$ is the electric voltage at the point $x\in \Omega$,
$\nu$ is the exterior unit normal vector of $\partial\Omega,$ $\p_\nu u$
is the normal derivative of $u$ on the boundary, the Robin coefficient $q(x)$ models
the corrosion (causing contact impedance on the boundary), and $f(x)$ is the 
external normal current through the boundary.

Assume also that we are given a function $q_0$ that is close to the true
Robin coefficient and that the solution of the equation
\begin{eqnarray} \label{ES problem 0}
 & &\Delta u_0(x)=0,\quad\ \hbox{in }x\in \Omega,\\
& &\p_\nu u_0+q_0u_0=0,\quad\ \hbox{on }x\in \Gamma,\\
& &\p_\nu u_0=f,\quad\ \hbox{on }x\in \partial \Omega\setminus \Gamma
\end{eqnarray} 
satisfies 
\begin{eqnarray} \label{ES problem A}
|u_0(x)|\ge c_0>0\quad \hbox{on $\overline \Gamma$}.
\end{eqnarray}
Note that a condition analogous to \eqref{ES problem A} is necessary to determine the coefficient $q_0$. For example, if $u_0|_\Gamma$ vanishes, then $q_0$ does not influence the data $u_0|_\Sigma$. 

When $q$ is close to $q_0$, the inverse problem is to determine $q$ from
the measurement $g^q:=u^q|_\Sigma$, where $\Sigma\subset \partial \Omega\setminus \overline \Gamma$
is an open, non-empty set. 
To study this inverse problem, we assume that $q_0$ is not identically zero, and that
the true Robin coefficient $q$ and the boundary current $f$ satisfy
%satisfy \footnote{The  the assumptions in reference Cheng-Choulli-Lin need to be checked.}
\begin{eqnarray} \label{ES problem assumptions1}
 & &q\in H^2_0(\Gamma),\quad q\ge 0,\\
  \label{ES problem assumptions2}
& &f\in C^{1,\alpha}(\p\Omega),\ 0<\alpha<1,\\
 \label{ES problem assumptions3}
& &\|q-q_0\|_{H^2(\p\Omega)}<\e_0,
\end{eqnarray} 
and $H^2_0(\Gamma)$ is the closure of $C^\infty_0(\Gamma)$ in the Sobolev space $H^2(\partial\Omega)$,
while $C^\infty_0(\Gamma)$ is the set of smooth, compactly supported functions on $\Gamma$.
The solution map of the inverse problem is
\begin{eqnarray} \label{ES IP solution map pre}
& &\mathcal F_0:u^q|_\Sigma\mapsto q|_\Gamma.
\end{eqnarray} 

\begin{proposition}\label{Prof Robin IP}
When $\e_0>0$ is small enough,
the solution map \eqref{ES IP solution map pre} admits a continuous nonlinear extension
\begin{eqnarray} \label{ES IP solution map}
 & &\mathcal F:L^2(\Sigma)\to L^2(\Gamma),\\
& &\mathcal F(u^q|_\Sigma)= q|_\Gamma
\end{eqnarray} 
that satisfies
\begin{eqnarray} \label{ES IP estimate}
 & &\|\mathcal F(g_1)-\mathcal F(g_2)\|_{L^2(\Gamma)}\leq \omega(\|g_1-g_2\|_{L^2(\Sigma)}),
 \end{eqnarray}
where 
 \begin{eqnarray}   \label{stability omega}
  \omega(t)=C_0\bigg(\log(1+\frac {C_1}t)\bigg)^{-1}
\end{eqnarray} 
for some constants $C_0,C_1$ depending on $\Omega,\Gamma,\Sigma$, $f$, and $q_0$. 
\end{proposition}

\begin{proof}
Let $\mathcal Y\subset H^2_0(\Gamma)$ be the set of functions
that satisfy conditions \eqref{ES problem assumptions1}-\eqref{ES problem assumptions3}.
Moreover, let $\mathcal X\subset L^2(\Sigma)$ be the set of the functions $u^q|_\Sigma$,
for which there exists $q$ satisfying \eqref{ES problem assumptions1}-\eqref{ES problem assumptions3}.

Using \cite[Theorem 6.31]{GilbardTrudinger_EllipticPDEsBook_1997} and the remark following it, we see that the solution $u^q\in C^{2,\alpha}(\overline \Omega)$, $0<\alpha<1/2$,
exists and is unique. Moreover, by applying \cite[Theorem 6.30]{GilbardTrudinger_EllipticPDEsBook_1997} to 
$v=u^q-u^{q_0}$ that satisfies
\begin{eqnarray} \label{ES problem v}
 & &\Delta v(x)=0,\quad\ \hbox{in }x\in \Omega,\\
& &\p_\nu v+qv=(q-q_0)u_0,\quad\ \hbox{on }x\in \Gamma,\\
& &\p_\nu v=0,\quad\ \hbox{on }x\in \partial \Omega\setminus \Gamma,
\end{eqnarray} 
we see that $u^q\in C^{2,\alpha}(\overline \Omega)$ 
depends continuously on $q\in C^{1,\alpha}_0(\Gamma).$
By the Sobolev embedding theorem, the identity map $H^2(\p\Omega)\to C^{1,\alpha}(\p\Omega)$ is bounded 
for $0<\alpha<1/2$, and thus, when $\e_0>0$ is small enough, it holds for all $q\in\mathcal Y$ that  
\begin{eqnarray}\label{eq c per 2}
|u^q(x)|>c_0/2>0,\quad\hbox{for all $x\in \overline\Gamma$.}
\end{eqnarray}
Then, the assumptions of \cite[Corollary 2.2]{Cheng-Choulli-Lin} are valid, and this result implies that
%\footnote{Cheng, Jin; Choulli, Mourad; Lin, Junshan,
%Stable determination of a boundary coefficient in an elliptic equation. 
%Math. Models Methods Appl. Sci. 18 (2008), no. 1, 107-123. } for any $g\in\mathcal X$ 
there is a unique $q\in \mathcal Y$ such that $g=u^q|_\Sigma$.
Thus, there is a well-defined map
 $\mathcal F_0:\mathcal X\to \mathcal Y$ defined by $\mathcal F_0(u^q|_\Sigma)=q|_\Gamma$.
 Moreover, by \cite[Corollary 2.2]{Cheng-Choulli-Lin}, for $q_1,q_2\in \mathcal Y$ 
 the functions $g_1=u^{q_1}|_\Sigma$ and $g_2=u^{q_2}|_\Sigma$
 satisfy 
 \begin{eqnarray} \label{ES IP estimate 2}
 & &\|\mathcal F_0(g_1)-\mathcal F_0(g_2)\|_{L^2(\Gamma)}\leq \omega(\|g_1-g_2\|_{L^2(\Sigma)}),
 \end{eqnarray}
where $\omega$ is given in \eqref{stability omega}.

Consider the sets $\mathcal X$ and $\mathcal Y$ as subsets of Lebesgue spaces, $\mathcal X\subset L^2(\Sigma)$
and $\mathcal Y\subset L^2(\Gamma)$.
Then, $L^2(\Sigma)$
and $L^2(\Gamma)$ are Hilbert spaces. The Benyamini-Lindenstrauss theorem, \cite[Theorem 1.2.12]{BenyaminiLindenstrauss_Book_Vol1_2000}, implies that
the map $\mathcal F_0:\mathcal X\to \mathcal Y$ admits a continuous extension 
$$
\mathcal F:L^2(\Sigma)\to L^2(\Gamma),
$$
to the whole vector space $L^2( \Sigma)$
that satisfies $\mathcal F|_{\mathcal X}=\mathcal F_0$
and has the same modulus of continuity $\omega$ as $\mathcal F_0$,
that is, the inequality \eqref{ES IP estimate} is valid.
\end{proof}

{\bf Remark A1. } We briefly explain the choice of $\omega(\cdot)$ as the modulus of continuity for $\mathcal F$ in Proposition \ref{Prof Robin IP}. 
Consider the case where $f$ is supported in $\Sigma$
and the boundary $\partial \Omega$ is real analytic. 
The results cited above 
on the solvability of the inverse problem, \cite[Corollary 2.2]{Cheng-Choulli-Lin}, are based on unique continuation for elliptic equations. Moreover, the determination of $q$ from the given data $g^q$ can be done in two steps. 
In the first step, we use the unique continuation theorem for harmonic functions, see \cite[Theorem 2.1]{Cheng-Choulli-Lin}, to determine $u^q$ in a neighborhood $V$ of $\Gamma$.

Recall that $\Delta u^q=0$ and $q\in H^2(\partial \Omega)\subset C^{1,\beta}(\partial \Omega)$ for any $0<\beta<\min(\alpha,1/2)$.
We then apply elliptic regularity. Since $q\ge 0$ is not identically zero, \cite[Theorem 6.31]{GilbardTrudinger_EllipticPDEsBook_1997} and the remark following it imply that the solution $u^q$ is unique and satisfies 
$$
u^q\in C^{2,\beta}(\overline \Omega)\subset H^{2+\beta}(\overline \Omega).
$$
%Then, by \cite[Theorem 8.8]{GilbardTrudinger_EllipticPDEsBook_1997},
%we see that $u^q\in H^2(\Omega)\subset C(\overline\Omega)$.
%
%
%Then, by
%\cite[Lemma 6.29]{GilbardTrudinger_EllipticPDEsBook_1997}, we see that $u^q|_V$ is in $C^{2,\alpha}
%(V)$
%in a neighborhood $V\subset \overline \Omega$ of $\Gamma$ in the relative topology of 
%$\overline \Omega\subset \mathbb R^2$. 
Next, we view $u^q$ as an element of $H^{2+\beta}(\overline \Omega)$.
By the trace theorem, the boundary values  
\begin{eqnarray}\label{boundary value 1}
& &u^q|_{\partial \Omega}\in 
H^{3/2+\beta}(\partial \Omega) \subset  C^{1,\beta}(\partial\Omega),\\
\label{boundary value 2}
& &\p_\nu u^q|_{\partial \Omega}\in 
H^{1/2+\beta}(\partial  \Omega) \subset  C^{0,\beta}(\partial \Omega),
\end{eqnarray}
depend continuously on $u^q\in H^{2+\beta}(\overline \Omega)$.
%In particular, $u^q|_\Gamma$ is in $C^{2,\beta}(\overline \Gamma)$
%and $\partial_\nu u^q|_\Gamma$ is in $C^{1,\beta}(\overline \Gamma)$.
The second step is to recover $q|_\Gamma$ using \eqref{eq c per 2},
\eqref{boundary value 1}, \eqref{boundary value 2}, and
the formula
$$
q=-\frac{\partial_\nu u^q|_\Gamma}{u^q|_\Gamma}\in C^{0,\beta}(\overline \Gamma)\subset L^2(\Gamma).
$$
Thus $g^q$ determines $q$ on $\Gamma$ uniquely.
 
As seen above,
the solution map $\mathcal F$ of the inverse problem can be written as the composition
of two maps, $\mathcal F=\mathcal G\circ \mathcal H$, where $\mathcal H:u^q|_{\Sigma}\mapsto u^q|_V$,
with $V\subset \overline \Omega$ a neighborhood of $\Gamma$ in the relative topology of
$\overline \Omega\subset \mathbb R^2$, and $\mathcal G: u^q|_V\mapsto q$. By using \cite[Theorem 2.1]{Cheng-Choulli-Lin},
%\cite[Theorem 8.8]{GilbardTrudinger_EllipticPDEsBook_1997}, 
and \cite[Theorem 1.2.12]{BenyaminiLindenstrauss_Book_Vol1_2000}, we see that the map 
$\mathcal H:\mathcal X\to H^{2+\beta}(V)$ extends to a continuous map 
$\mathcal H:L^2(\Sigma)\to H^{2+\beta}(V)$. The 
general non-stability results for unique continuation, proven by Koch, Ruland, and Salo, see 
 \cite[Theorem 4.3]{Koch}
 %\footnote{
%Koch, Herbert ; Ruland, Angkana ; Salo, Mikko,
%On instability mechanisms for inverse problems.Ars Inven. Anal. 2021, Paper No. 7, 93 pp. } 
show that the optimal modulus of continuity for $\mathcal H:L^2(\Sigma)\to H^{2+\beta}(V)$ is logarithmic.
This motivates using a logarithmic modulus of continuity for the map $\mathcal F=\mathcal G\circ \mathcal H$.

\section{Neural operator as (distributed) mixtures of neural operators}\label{ex:Centralized_as_distributed}
% \paragraph{Neural operator as (distributed) mixtures of neural operators}\label{ex:Centralized_as_distributed}
% \label{defn:realization}
\noindent
In this section, we show that a standard neural operator can be identified as a mixture of neural operators.
Suppose that the assumptions in \Cref{def:neural-operator-v2} hold. Let $G \in \mathcal{NO}^{\tanh}_{N,W,L,\Delta,\boldsymbol d}$ so that
\[
G : (H^{s_1}(D_1)^{d_{in}}, \|\cdot\|_{L^2(D_1)^{d_{in}}})
   \to (H^{s_2}(D_2)^{d_{out}}, \|\cdot\|_{L^2(D_2)^{d_{out}}}).
\]
Choose an arbitrary but fixed $f_{1,1} \in H^{s_1}(D_1)^{d_{in}}$ and define the trivial tree
\[
\mathcal{T}_{\mathrm{trivial}} = (\{f_{1,1}\}, \varnothing, f_{1,1}),
\]
which consists solely of the root $f_{1,1}$ and no other nodes. 
The pair $(\mathcal{T}_{\mathrm{trivial}}, \{G\})$ vacuously satisfies the conditions in 
\Cref{s:Prelim__ss:DistributedHypotheses}.

Moreover, 
$$
\operatorname{Activ-Cpl}((\mathcal{T}_{\mathrm{trivial}}, \{G\}))
=
\operatorname{Dist-Cpl}((\mathcal{T}_{\mathrm{trivial}}, \{G\}));
$$
that is, the active complexity coincides with the total complexity because there is only one computational node: the unique leaf, which is also the root. This single expert is responsible for all inputs in $H^{s_1}(D_1)^{d_{in}}$.

The routing depth is trivial: there is no branch decision, only the root/leaf evaluation.
Furthermore, there exist infinitely many such pairs $(\mathcal{T}_{\mathrm{trivial}}, \{G\})$, since the choice of the root $f_{1,1}$ is arbitrary.

However, there is exactly one realization of $(\mathcal{T}_{\mathrm{trivial}}, \{G\})$, and it coincides with the classical centralized neural operator. Henceforth, we identify the classical neural operator with this trivial distributed counterpart.

%\bibliographystyle{siamplain}
%\bibliography{Bookkeaping/5_Refs}

\begin{thebibliography}{48}
\providecommand{\natexlab}[1]{#1}
\providecommand{\url}[1]{\texttt{#1}}
\expandafter\ifx\csname urlstyle\endcsname\relax
  \providecommand{\doi}[1]{doi: #1}\else
  \providecommand{\doi}{doi: \begingroup \urlstyle{rm}\Url}\fi

\bibitem[Acciaio et~al.(2023)Acciaio, Kratsios, and
  Pammer]{acciaio2023designing}
Beatrice Acciaio, Anastasis Kratsios, and Gudmund Pammer.
\newblock Designing universal causal deep learning models: The geometric
  (hyper) transformer.
\newblock \emph{Mathematical Finance}, 2023.

\bibitem[Adcock et~al.(2025)Adcock, Brugiapaglia, Dexter, and
  Moraga]{adcock2022near}
Ben Adcock, Simone Brugiapaglia, Nick Dexter, and Sebastian Moraga.
\newblock Near-optimal learning of banach-valued, high-dimensional functions
  via deep neural networks.
\newblock \emph{Neural Networks}, 181:\penalty0 106761, 2025.
\newblock \doi{10.1016/j.neunet.2024.106761}.

\bibitem[Andrade-Loarca et~al.(2023)Andrade-Loarca, Bacho, Hege, and
  Kutyniok]{andrade2023poissonnet}
Hector Andrade-Loarca, Aras Bacho, Julius Hege, and Gitta Kutyniok.
\newblock Poissonnet: Resolution-agnostic 3d shape reconstruction using fourier
  neural operators.
\newblock \emph{arXiv preprint arXiv:2308.01766}, 2023.

\bibitem[Barham et~al.(2022)Barham, Chowdhery, Dean, Ghemawat, Hand, Hurt,
  Isard, Lim, Pang, Roy, et~al.]{barham2022pathways}
Paul Barham, Aakanksha Chowdhery, Jeff Dean, Sanjay Ghemawat, Steven Hand,
  Daniel Hurt, Michael Isard, Hyeontaek Lim, Ruoming Pang, Sudip Roy, et~al.
\newblock Pathways: Asynchronous distributed dataflow for ml.
\newblock \emph{Proceedings of Machine Learning and Systems}, 4:\penalty0
  430--449, 2022.

\bibitem[Barron(1993)]{barron1993universal}
Andrew~R. Barron.
\newblock Universal approximation bounds for superpositions of a sigmoidal
  function.
\newblock \emph{IEEE Transactions on Information Theory}, 39\penalty0
  (3):\penalty0 930--945, 1993.
\newblock \doi{10.1109/18.256500}.

\bibitem[Bartolucci et~al.(2024)Bartolucci, de~Bezenac, Raonic, Molinaro,
  Mishra, and Alaifari]{bartolucci2024representation}
Francesca Bartolucci, Emmanuel de~Bezenac, Bogdan Raonic, Roberto Molinaro,
  Siddhartha Mishra, and Rima Alaifari.
\newblock Representation equivalent neural operators: a framework for
  alias-free operator learning.
\newblock \emph{Advances in Neural Information Processing Systems}, 36, 2024.

\bibitem[Batlle et~al.(2024)Batlle, Darcy, Hosseini, and
  Owhadi]{batlle2024kernel}
Pau Batlle, Matthieu Darcy, Bamdad Hosseini, and Houman Owhadi.
\newblock Kernel methods are competitive for operator learning.
\newblock \emph{Journal of Computational Physics}, 496:\penalty0 112549, 2024.
\newblock \doi{10.1016/j.jcp.2023.112549}.

\bibitem[Benyamini and
  Lindenstrauss(2000)]{BenyaminiLindenstrauss_Book_Vol1_2000}
Yoav Benyamini and Joram Lindenstrauss.
\newblock \emph{Geometric nonlinear functional analysis. {V}ol. 1}, volume~48
  of \emph{American Mathematical Society Colloquium Publications}.
\newblock American Mathematical Society, Providence, RI, 2000.
\newblock ISBN 0-8218-0835-4.
\newblock \doi{10.1090/coll/048}.
\newblock URL \url{https://doi.org/10.1090/coll/048}.

\bibitem[Birman and Solomyak(1967)]{birman1967piecewise}
Mikhail~Shlemovich Birman and Mikhail~Zakharovich Solomyak.
\newblock Piecewise-polynomial approximations of functions of the classes
  w\_p\^{}$\alpha$.
\newblock \emph{Matematicheskii Sbornik}, 115\penalty0 (3):\penalty0 331--355,
  1967.

\bibitem[Chen and Chen(1995)]{Chen2_OG_NeuralOperators_IEEE_1995}
Tianping Chen and Hong Chen.
\newblock Universal approximation to nonlinear operators by neural networks
  with arbitrary activation functions and its application to dynamical systems.
\newblock \emph{IEEE Transactions on Neural Networks}, 6\penalty0 (4):\penalty0
  911--917, 1995.
\newblock \doi{10.1109/72.392253}.

\bibitem[Cheng et~al.(2008)Cheng, Choulli, and Lin]{Cheng-Choulli-Lin}
Jin Cheng, Mourad Choulli, and Junshan Lin.
\newblock Stable determination of a boundary coefficient in an elliptic
  equation.
\newblock \emph{Math. Models Methods Appl. Sci.}, 18\penalty0 (1):\penalty0
  107--123, 2008.
\newblock ISSN 0218-2025.
\newblock \doi{10.1142/S0218202508002620}.
\newblock URL \url{https://doi.org/10.1142/S0218202508002620}.

\bibitem[de~Hoop et~al.(2022)de~Hoop, Lassas, and Wong]{de2022deep}
Maarten~V de~Hoop, Matti Lassas, and Christopher~A Wong.
\newblock Deep learning architectures for nonlinear operator functions and
  nonlinear inverse problems.
\newblock \emph{Mathematical Statistics and Learning}, 4\penalty0 (1):\penalty0
  1--86, 2022.

\bibitem[Fedus et~al.(2022)Fedus, Zoph, and Shazeer]{fedus2022switch}
William Fedus, Barret Zoph, and Noam Shazeer.
\newblock Switch transformers: Scaling to trillion parameter models with simple
  and efficient sparsity.
\newblock \emph{Journal of Machine Learning Research}, 23\penalty0
  (120):\penalty0 1--39, 2022.

\bibitem[Furuya et~al.(2023)Furuya, Puthawala, Lassas, and
  de~Hoop]{furuya2023globally}
Takashi Furuya, Michael Puthawala, Matti Lassas, and Maarten~V de~Hoop.
\newblock Globally injective and bijective neural operators.
\newblock \emph{Thirty-seventh Conference on Neural Information Processing
  Systems}, 2023.

\bibitem[Galimberti et~al.(2022)Galimberti, Kratsios, and
  Livieri]{galimberti2022designing}
Luca Galimberti, Anastasis Kratsios, and Giulia Livieri.
\newblock Designing universal causal deep learning models: The case of
  infinite-dimensional dynamical systems from stochastic analysis.
\newblock \emph{arXiv preprint arXiv:2210.13300}, 2022.

\bibitem[Gilbarg and Trudinger(1977)]{GilbardTrudinger_EllipticPDEsBook_1997}
David Gilbarg and Neil~S. Trudinger.
\newblock \emph{Elliptic partial differential equations of second order},
  volume Vol. 224 of \emph{Grundlehren der Mathematischen Wissenschaften}.
\newblock Springer-Verlag, Berlin-New York, 1977.
\newblock ISBN 3-540-08007-4.

\bibitem[Herrmann et~al.(2022)Herrmann, Opschoor, and
  Schwab]{herrmann2022constructive}
Lukas Herrmann, Joost~AA Opschoor, and Christoph Schwab.
\newblock Constructive deep relu neural network approximation.
\newblock \emph{Journal of Scientific Computing}, 90\penalty0 (2):\penalty0 75,
  2022.

\bibitem[Herrmann et~al.(2024)Herrmann, Schwab, and Zech]{herrmann2022neural}
Lukas Herrmann, Christoph Schwab, and Jakob Zech.
\newblock Neural and spectral operator surrogates: unified construction and
  expression rate bounds.
\newblock \emph{Advances in Computational Mathematics}, 50\penalty0 (72), 2024.
\newblock \doi{10.1007/s10444-024-10171-2}.

\bibitem[Jiang et~al.(2024)Jiang, Sablayrolles, Roux, Mensch, Savary, Bamford,
  Chaplot, Casas, Hanna, Bressand, et~al.]{jiang2024mixtral}
Albert~Q Jiang, Alexandre Sablayrolles, Antoine Roux, Arthur Mensch, Blanche
  Savary, Chris Bamford, Devendra~Singh Chaplot, Diego de~las Casas, Emma~Bou
  Hanna, Florian Bressand, et~al.
\newblock Mixtral of experts.
\newblock \emph{arXiv preprint arXiv:2401.04088}, 2024.

\bibitem[Jiao et~al.(2023)Jiao, Lai, Lu, Wang, Yang, and Yang]{jiao2023deep}
Yuling Jiao, Yanming Lai, Xiliang Lu, Fengru Wang, Jerry~Zhijian Yang, and
  Yuanyuan Yang.
\newblock Deep neural networks with relu-sine-exponential activations break
  curse of dimensionality in approximation on h{\"o}lder class.
\newblock \emph{SIAM Journal on Mathematical Analysis}, 55\penalty0
  (4):\penalty0 3635--3649, 2023.

\bibitem[Koch et~al.(2021)Koch, R\"{u}land, and Salo]{Koch}
Herbert Koch, Angkana R\"{u}land, and Mikko Salo.
\newblock On instability mechanisms for inverse problems.
\newblock \emph{Ars Inven. Anal.}, pages Paper No. 7, 93, 2021.

\bibitem[Korolev(2022)]{korolev2022two}
Yury Korolev.
\newblock Two-layer neural networks with values in a banach space.
\newblock \emph{SIAM Journal on Mathematical Analysis}, 54\penalty0
  (6):\penalty0 6358--6389, 2022.

\bibitem[Kovachki et~al.(2021)Kovachki, Lanthaler, and
  Mishra]{KovachkiLanthalerMishra_UniFNO_JMLR_2021}
Nikola Kovachki, Samuel Lanthaler, and Siddhartha Mishra.
\newblock On universal approximation and error bounds for {F}ourier neural
  operators.
\newblock \emph{J. Mach. Learn. Res.}, 22:\penalty0 Paper No. [290], 76, 2021.
\newblock ISSN 1532-4435,1533-7928.

\bibitem[Kovachki et~al.(2023)Kovachki, Li, Liu, Azizzadenesheli, Bhattacharya,
  Stuart, and Anandkumar]{kovachki2021neural}
Nikola Kovachki, Zongyi Li, Burigede Liu, Kamyar Azizzadenesheli, Kaushik
  Bhattacharya, Andrew Stuart, and Anima Anandkumar.
\newblock Neural operator: Learning maps between function spaces with
  applications to pdes.
\newblock \emph{Journal of Machine Learning Research}, 24\penalty0
  (89):\penalty0 1--97, 2023.

\bibitem[Kratsios and Papon(2022)]{kratsios2022universal}
Anastasis Kratsios and L{\'e}onie Papon.
\newblock Universal approximation theorems for differentiable geometric deep
  learning.
\newblock \emph{The Journal of Machine Learning Research}, 23\penalty0
  (1):\penalty0 8896--8968, 2022.

\bibitem[Lanthaler(2023)]{PCANetErrorBounds_JMLR_2023}
Samuel Lanthaler.
\newblock Operator learning with pca-net: upper and lower complexity bounds.
\newblock \emph{Journal of Machine Learning Research}, 24\penalty0
  (318):\penalty0 1--67, 2023.
\newblock URL \url{http://jmlr.org/papers/v24/23-0478.html}.

\bibitem[Lanthaler and Stuart(2023)]{lanthaler2023curse}
Samuel Lanthaler and Andrew~M Stuart.
\newblock The curse of dimensionality in operator learning.
\newblock \emph{arXiv preprint arXiv:2306.15924}, 2023.

\bibitem[Lanthaler and Stuart(2026)]{lanthaler2025parametric}
Samuel Lanthaler and Andrew~M. Stuart.
\newblock The parametric complexity of operator learning.
\newblock \emph{IMA Journal of Numerical Analysis}, 46\penalty0 (2):\penalty0
  647--712, 2026.
\newblock \doi{10.1093/imanum/draf028}.
\newblock Published online 2025.

\bibitem[Lanthaler et~al.(2023)Lanthaler, Rusch, and
  Mishra]{lanthaler2023neural}
Samuel Lanthaler, T.~Konstantin Rusch, and Siddhartha Mishra.
\newblock Neural oscillators are universal.
\newblock In \emph{Thirty-seventh Conference on Neural Information Processing
  Systems}, 2023.
\newblock URL \url{https://openreview.net/forum?id=QGQsOZcQ2H}.

\bibitem[Lara~Benitez et~al.(2025)Lara~Benitez, Hegazy, Guo, Dokmani{\'c},
  Mahoney, and de~Hoop]{larabenitez2025neurde}
J.~Antonio Lara~Benitez, Kareem Hegazy, Junyi Guo, Ivan Dokmani{\'c},
  Michael~W. Mahoney, and Maarten~V. de~Hoop.
\newblock Neural equilibria for long-term prediction of nonlinear conservation
  laws.
\newblock \emph{arXiv preprint arXiv:2501.06933}, 2025.

\bibitem[Lepikhin et~al.(2021)Lepikhin, Lee, Xu, Chen, Firat, Huang, Krikun,
  Shazeer, and Chen]{lepikhin2021gshard}
Dmitry Lepikhin, HyoukJoong Lee, Yuanzhong Xu, Dehao Chen, Orhan Firat, Yanping
  Huang, Maxim Krikun, Noam Shazeer, and Zhifeng Chen.
\newblock {\{}GS{\}}hard: Scaling giant models with conditional computation and
  automatic sharding.
\newblock In \emph{International Conference on Learning Representations}, 2021.
\newblock URL \url{https://openreview.net/forum?id=qrwe7XHTmYb}.

\bibitem[Li et~al.(2021)Li, Kovachki, Azizzadenesheli, liu, Bhattacharya,
  Stuart, and Anandkumar]{li2021fourier}
Zongyi Li, Nikola~Borislavov Kovachki, Kamyar Azizzadenesheli, Burigede liu,
  Kaushik Bhattacharya, Andrew Stuart, and Anima Anandkumar.
\newblock Fourier neural operator for parametric partial differential
  equations.
\newblock In \emph{International Conference on Learning Representations}, 2021.
\newblock URL \url{https://openreview.net/forum?id=c8P9NQVtmnO}.

\bibitem[Lorentz et~al.(1996)Lorentz, Golitschek, and
  Makovoz]{LorentzGoliteschekMakovoz_1996_CAAdvProbsBook}
George~G. Lorentz, Manfred~v. Golitschek, and Yuly Makovoz.
\newblock \emph{Constructive approximation}, volume 304 of \emph{Grundlehren
  der mathematischen Wissenschaften [Fundamental Principles of Mathematical
  Sciences]}.
\newblock Springer-Verlag, Berlin, 1996.
\newblock ISBN 3-540-57028-4.
\newblock \doi{10.1007/978-3-642-60932-9}.
\newblock URL \url{https://doi.org/10.1007/978-3-642-60932-9}.
\newblock Advanced problems.

\bibitem[Lu et~al.(2021)Lu, Jin, Pang, Zhang, and Karniadakis]{lu2021learning}
Lu~Lu, Pengzhan Jin, Guofei Pang, Zhongqiang Zhang, and George~Em Karniadakis.
\newblock Learning nonlinear operators via deeponet based on the universal
  approximation theorem of operators.
\newblock \emph{Nature machine intelligence}, 3\penalty0 (3):\penalty0
  218--229, 2021.

\bibitem[Marcati and Schwab(2023)]{marcati2023exponential}
Carlo Marcati and Christoph Schwab.
\newblock Exponential convergence of deep operator networks for elliptic
  partial differential equations.
\newblock \emph{SIAM Journal on Numerical Analysis}, 61\penalty0 (3):\penalty0
  1513--1545, 2023.

\bibitem[McCabe et~al.(2023)McCabe, Harrington, Subramanian, and
  Brown]{mccabe2023towards}
Michael McCabe, Peter Harrington, Shashank Subramanian, and Jed Brown.
\newblock Towards stability of autoregressive neural operators.
\newblock \emph{Transactions on Machine Learning Research}, 2023.
\newblock ISSN 2835-8856.
\newblock URL \url{https://openreview.net/forum?id=RFfUUtKYOG}.

\bibitem[Molinaro et~al.(2023)Molinaro, Yang, Engquist, and
  Mishra]{pmlr-v202-molinaro23a}
Roberto Molinaro, Yunan Yang, Bj\"{o}rn Engquist, and Siddhartha Mishra.
\newblock Neural inverse operators for solving {PDE} inverse problems.
\newblock In Andreas Krause, Emma Brunskill, Kyunghyun Cho, Barbara Engelhardt,
  Sivan Sabato, and Jonathan Scarlett, editors, \emph{Proceedings of the 40th
  International Conference on Machine Learning}, volume 202 of
  \emph{Proceedings of Machine Learning Research}, pages 25105--25139. PMLR,
  23--29 Jul 2023.
\newblock URL \url{https://proceedings.mlr.press/v202/molinaro23a.html}.

\bibitem[Nelsen and Stuart(2024)]{nelsen2024operator}
Nicholas~H. Nelsen and Andrew~M. Stuart.
\newblock Operator learning using random features: A tool for scientific
  computing.
\newblock \emph{SIAM Review}, 66\penalty0 (3):\penalty0 535--571, 2024.
\newblock \doi{10.1137/24M1648703}.

\bibitem[Petersen and Voigtlaender(2018)]{petersen2018optimal}
Philipp Petersen and Felix Voigtlaender.
\newblock Optimal approximation of piecewise smooth functions using deep relu
  neural networks.
\newblock \emph{Neural Networks}, 108:\penalty0 296--330, 2018.

\bibitem[Pinkus(2012)]{pinkus2012n}
Allan Pinkus.
\newblock \emph{N-widths in Approximation Theory}, volume~7.
\newblock Springer Science \& Business Media, 2012.

\bibitem[Raonic et~al.(2024)Raonic, Molinaro, De~Ryck, Rohner, Bartolucci,
  Alaifari, Mishra, and de~B{\'e}zenac]{raonic2024convolutional}
Bogdan Raonic, Roberto Molinaro, Tim De~Ryck, Tobias Rohner, Francesca
  Bartolucci, Rima Alaifari, Siddhartha Mishra, and Emmanuel de~B{\'e}zenac.
\newblock Convolutional neural operators for robust and accurate learning of
  pdes.
\newblock \emph{Advances in Neural Information Processing Systems}, 36, 2024.

\bibitem[Shazeer et~al.(2017)Shazeer, Mirhoseini, Maziarz, Davis, Le, Hinton,
  and Dean]{shazeer2017outrageously}
Noam Shazeer, Azalia Mirhoseini, Krzysztof Maziarz, Andy Davis, Quoc Le,
  Geoffrey Hinton, and Jeff Dean.
\newblock Outrageously large neural networks: The sparsely-gated
  mixture-of-experts layer.
\newblock \emph{arXiv preprint arXiv:1701.06538}, 2017.

\bibitem[Shen et~al.(2021)Shen, Yang, and Zhang]{shen2021deep}
Zuowei Shen, Haizhao Yang, and Shijun Zhang.
\newblock Deep network with approximation error being reciprocal of width to
  power of square root of depth.
\newblock \emph{Neural Computation}, 33\penalty0 (4):\penalty0 1005--1036,
  2021.

\bibitem[Yarotsky(2018)]{yarotsky2018optimal}
Dmitry Yarotsky.
\newblock Optimal approximation of continuous functions by very deep relu
  networks.
\newblock In \emph{Conference on learning theory}, pages 639--649. PMLR, 2018.

\bibitem[Yarotsky(2021)]{yarotsky2021elementary}
Dmitry Yarotsky.
\newblock Elementary superexpressive activations.
\newblock In \emph{International Conference on Machine Learning}, pages
  11932--11940. PMLR, 2021.

\bibitem[Yarotsky and Zhevnerchuk(2020)]{yarotsky2020phase}
Dmitry Yarotsky and Anton Zhevnerchuk.
\newblock The phase diagram of approximation rates for deep neural networks.
\newblock \emph{Advances in neural information processing systems},
  33:\penalty0 13005--13015, 2020.

\bibitem[Zhang et~al.(2023)Zhang, Lu, and Zhao]{zhang2023deep}
Shijun Zhang, Jianfeng Lu, and Hongkai Zhao.
\newblock Deep network approximation: Beyond relu to diverse activation
  functions.
\newblock \emph{arXiv preprint arXiv:2307.06555}, 2023.

\bibitem[Zhou et~al.(2022)Zhou, Lei, Liu, Du, Huang, Zhao, Dai, Le, Laudon,
  et~al.]{zhou2022mixture}
Yanqi Zhou, Tao Lei, Hanxiao Liu, Nan Du, Yanping Huang, Vincent Zhao, Andrew~M
  Dai, Quoc~V Le, James Laudon, et~al.
\newblock Mixture-of-experts with expert choice routing.
\newblock \emph{Advances in Neural Information Processing Systems},
  35:\penalty0 7103--7114, 2022.

\end{thebibliography}

%\end{document}

\end{document}